\documentclass[journal]{IEEEtran}
\usepackage{amsmath,amsfonts,amssymb,amsthm,mathrsfs}
\usepackage{algorithm}
\usepackage{algpseudocode}
\usepackage{array}
\usepackage[caption=false,font=normalsize,labelfont=sf,textfont=sf]{subfig}
\usepackage{textcomp}
\usepackage{stfloats}
\usepackage{url}
\usepackage{verbatim}
\usepackage{graphicx}
\usepackage{cite}
\usepackage{multirow}
\usepackage{threeparttable}
\usepackage{supertabular}
\usepackage[table,xcdraw]{xcolor}
\usepackage{hyperref}
\usepackage{orcidlink}
\hypersetup{colorlinks=true,
	linkcolor=blue,
	anchorcolor=blue,
	citecolor=blue}
\usepackage{lipsum}
\usepackage{makecell}
\usepackage{pifont}
\usepackage{nomencl}
\usepackage{siunitx}
\makenomenclature

\newtheorem{theorem}{Theorem}

\newtheorem{proposition}{Proposition}
\newtheorem{lemma}{Lemma}

\newtheorem{corollary}{Corollary}

\newcommand{\mb}[1]{\mathbf{#1}} 
\newcommand{\bs}[1]{\boldsymbol{#1}}
\newcommand{\mr}[1]{\mathrm{#1}} 
\newcommand{\mc}[1]{\mathcal{#1}}

\newcommand{\pt}{\partial}
\newcommand{\mrd}{\mathrm{d}}

\newcommand{\lt}{\left(}
\newcommand{\rt}{\right)}
\DeclareMathOperator{\tr}{tr}

\begin{document}
\title{Enhancing Kinematic Performances of Soft Continuum Robots for Magnetic Actuation}

\author{
Zhiwei Wu$^{\orcidlink{https://orcid.org/0000-0002-3957-3063}}$,~\IEEEmembership{Graduate Student Member,~IEEE}, Jiahao~Luo$^{\orcidlink{0009-0003-3197-1782}}$, Siyi~Wei$^{\orcidlink{0009-0005-9967-4025}}$, and Jinhui~Zhang$^{\orcidlink{0000-0002-2405-894X}}$.
\thanks{
Z. Wu, J. Luo, S. Wei, and J. Zhang are with the school of automation, Beijing Institute of Technology, Beijing, China. (e-mail: zhiweiwu.cn@outlook.com, luojiahao.edu@outlook.com, siyiwei@bit.edu.cn).
}
\thanks{
Corresponding author: Jinhui Zhang (e-mail: zhangjinh@bit.edu.cn).
}
}



\maketitle
\begin{abstract}
Soft continuum robots achieve complex deformation through elastic equilibrium, making their reachable motions governed jointly by structural design and actuation-induced mechanics. 
This work develops a general formulation that integrates equilibrium computation with kinematic performances by evaluating Riemannian Jacobian spectra on the equilibrium manifold shaped by internal/external loading. 
The resulting framework yields a global performance functional that directly links structural parameters, actuation inputs, and the induced configuration space geometry.
We apply this general framework to magnetic actuation.
Analytical characterization is obtained under weak uniform fields, revealing optimal placement and orientation of the embedded magnet with invariant scale properties.
To address nonlinear deformation and spatially varying fields, a two-level optimization algorithm is developed that alternates between energy based equilibrium search and gradient based structural updates.
Simulations and physical experiments across uniform field, dipole field, and multi-magnet configurations demonstrate consistent structural tendencies: aligned moments favor distal or mid-distal solutions through constructive torque amplification, whereas opposing moments compress optimal designs toward proximal regions due to intrinsic cancellation zones. 
\end{abstract}

\begin{IEEEkeywords}
Soft continuum robots, kinematic performances, structural optimization, magnetic actuation.
\end{IEEEkeywords}

\section{Introduction}

\IEEEPARstart{S}{oft} continuum robots have gained growing attention for tasks involving compliant interaction, dexterous access, and safe manipulation in complex or confined environments. 
Their ability to realize smooth, multi-segment deformation without rigid joints supports applications in minimally invasive navigation, inspection, and human-centered tasks. 
Recent advancements have significantly improved perception, control, and structural adaptability in soft continuum systems. 
Examples include proprioceptive designs capable of simultaneous shape estimation and contact detection \cite{Wang2024Sensing}, learning based deformation estimation under multi-point interaction \cite{Zhang2025Deformation}, optimal control strategies for rotation-coupled bending \cite{Zhao2024Controller}, and variable curvature mechanisms that enhance configurability \cite{Wang2024Soft}. 
These developments collectively broaden the functional space of soft continuum manipulation and motivate new models that can capture their nonlinear, compliant behavior.

Magnetic actuation has emerged as a particularly promising approach for soft continuum robots. 
Compared with tendon-driven or pneumatic mechanisms, magnetic actuation allows wireless control, allows miniaturization to submillimeter scales, and reduces onboard complexity. 
Recent studies demonstrate substantial gains in workspace, dexterity, and navigation capability across diverse magnetic architectures. 
Robots with distributed magnetization patterns have shown enhanced deformation diversity \cite{Huang2024Design}, movable opposite magnets have enabled workspace expansion \cite{Park2024Workspace}, and rotating field-assisted designs have achieved buckling free insertion in tortuous lumens \cite{Chathuranga2024Assisted}. 
High dexterity helical magnetic robots further extend steerability in endovascular navigation \cite{Dreyfus2024Dexterous}. 
In parallel, modeling frameworks for magnetic and elastic coupling have evolved rapidly. 
Analytical and semi-analytical models have been introduced for large deformation magnetoelastic behavior, including kinetostatic formulations capable of capturing multistage bending under complex dipole fields \cite{Zhang2025Kinetostatics}, modular magnetization profile evaluation \cite{Cao2025Magnetic}, and optimized structural patterns for enhanced deformability. 
Additional works explore interactions between magneto elasticity, external loading, and nonlinear geometries, providing foundations for real-time control and structural tuning. 
Together, these contributions highlight the importance of accurate equilibrium modeling and structural design for improving kinematic performance.

However, several core challenges remain unresolved. 
Magnetic soft continuum robots do not operate on a fixed kinematic map. 
Their attainable configurations are determined by the equilibrium manifold produced by the interaction between elastic restoring forces and magnetic loading. 
This manifold is highly nonlinear and sensitive to magnet placement, field variation, torque cancellation, and bifurcation events. 
Existing design methods often rely on task-specific heuristics \cite{Park2024Workspace}, empirical tuning of magnet distributions \cite{Huang2024Design}, data-driven control \cite{Tang2024Learning,Tang2025Learning}, or simplified kinematic approximations that overlook the intrinsic coupling between mechanics and actuation. 
Learning-based or adaptive controllers \cite{Kang2025Adaptive,Wang2024Sensing} enhance robustness, yet they offer limited insight into the structural sources underlying the equilibrium geometry.
Moreover, commonly used performance indices such as manipulability and dexterity, being defined in Euclidean coordinates, may only partially capture the curvature of the equilibrium manifold under magnetic loading.
As a result, structural optimization and performance evaluation are often decoupled, limiting the ability to systematically ensure actuation feasibility, stability, and kinematic quality.

These gaps motivate the need for a unified formulation that couples equilibrium with kinematic performance. 
Since the robot’s motion is realized only through equilibrium configurations, performance should be evaluated on the manifold shaped by internal and external loading. 
It further requires a structural optimization scheme capable of handling nonlinear magnetoelastic interactions, diverse field geometries, and multi-magnet architectures recently explored in continuum and catheter-like robots \cite{Francescon2025Closed,Park2024Workspace,Chathuranga2024Assisted}.

In this study, we develop a unified framework and demonstrate its theoretical and practical utility across homogeneous field and dipole field actuation. 
We leverage Riemannian Jacobian spectra to characterize manipulability and dexterity on the equilibrium manifold and integrate these indices with an energy-based equilibrium computation. This formulation enables physically consistent performance evaluation and reveals structural trends that persist across actuation regimes, including robots with modular magnetization segments \cite{Cao2025Magnetic}, multi-magnet bodies \cite{Huang2024Design}, and clinically oriented continuum tools \cite{VanLewen2025Real,Xu2025Automatic}.
The main contributions of this work are summarized as follows:
\begin{enumerate}
    
    \item We establish a unified structural optimization framework that couples magnetic energy-based equilibrium analysis with Riemannian Jacobian spectral evaluation on the equilibrium manifold. 
    By formulating the equilibrium set as an actuation-dependent submanifold and computing kinematic performance on its pullback metric. 
    
    \item We derive analytical results for soft continuum robots actuated by weak uniform magnetic fields, including closed-form optimality conditions for magnet placement and orientation. 
    These theoretical characterizations reveal scale-invariant structural principles that are further validated through numerical optimization.
    
    \item We extend the proposed framework to general dipole-generated magnetic fields and apply the full two-level optimization scheme to multi-magnet configurations. 
    Comprehensive numerical studies and corresponding physical experiments demonstrate consistent trends in equilibrium behavior and kinematic performance across diverse actuation paths.
\end{enumerate}

The remainder of this paper is organized as follows. 
Sec.~\ref{sec:II}  presents the modeling framework, including the pseudo-rigid-body representation, magnetic torque formulation, and energy-based equilibrium analysis. 
Sec.~\ref{sec:III} defines the kinematic performance metrics and formulates the structural optimization problem. 
Sec.~\ref{sec:IV} introduces the analytical and numerical optimization methods. 
Sec.~\ref{sec:V} provides case studies and numerical validations under uniform and general magnetic fields. 
Finally, Sec.~\ref{sec:VI} discusses the implications of the proposed framework and concludes the paper.
Throughout the technical description, vectors and matrices are defined using lowercase and uppercase bold symbols, respectively, whereas scalars are not bolded.

\section{Kinematics of Soft Continuum Robots}
\label{sec:II}
\subsection{Kinematic Mapping}

\begin{figure}[t]
    \centering
    \includegraphics[width=\columnwidth]{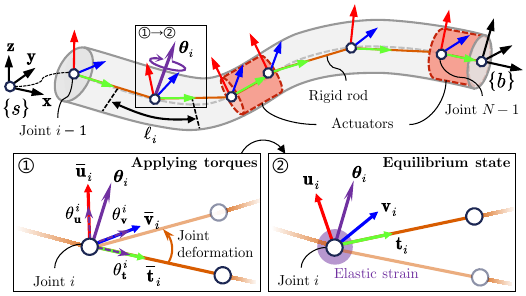}
    \caption{Illustration of the basic concept of the PRB model on describing the SCR. 
    The soft body is discretized into rigid rods connected by joints with local frames. 
    The configuration is represented by the joint variables $\theta_\mb{t}^i$, $\theta_\mb{u}^i$, and $\theta_\mb{v}^i$.
    The system reaches an equilibrium state (\ding{173}) when the torques of the actuators (\ding{172}) are balanced by the elastic restoring moments, and the joint deformation is stored as elastic strain.
    }
    \label{fig:configurationspace}
\end{figure}

We consider SCRs with slender structures, whose length is significantly greater than their radius. 
Certain well-founded assumptions contribute to simplifying the model \cite{oreillyModelingNonlinearProblems2017}, such as the rod remaining free of shear deformation and its cross-section always being orthogonal to the centerline. 
Additionally, some classical PRB models\cite{roesthuisSteeringMultisegmentContinuum2016} neglect both twisting and stretching of the continuum. 
In contrast, we assume only that the continuum is inextensible, with twisting characterized by extended variables \cite{pittiglio2023closed}. 

Within the PRB model framework, the SCR with length $L$ is described by a series of rigid rods. 
As shown in \autoref{fig:configurationspace}, $N$ spherical joints are distributed along the centerline of the SCR at points of interest, connected by $N$ rigid rods with lengths $l_i, i=0,\dots,N-1$. 
We use the Voronoi region of the joint to scale the potential energies, which is defined as $\ell_i=(l_{i-1}+l_{i})/2,i=0,\dots,N-1$. An orthogonal local frame $\{\bar{\mb{t}}_i,\bar{\mb{u}}_i,\bar{\mb{v}}_i=\bar{\mb{t}}_i\times\bar{\mb{u}}_i\}$ is attached to each joint, with overbars indicating that the SCR is in a straightenedl reference state, and the tangent vector $\bar{\mb{t}}_i$ remains aligned with the direction of the successive rod. 
The local frame is used to determine the rotation axis for each joint, expressed as $\bs\theta_i=\theta_{\mb{t}}^i\bar{\mb{t}}+\theta_{\mb{u}}^i\bar{\mb{u}}+\theta_{\mb{v}}^i\bar{\mb{v}}\in\mathbb{R}^3$, where $\theta_{\mb{t}}^i,\theta_{\mb{v}}^i,\theta_{\mb{u}}^i\in\mathbb{R}$ are three independent joint variables. 
The associated rotation matrix is defined as $\mb{R}_i(\bs\theta_i)=\mb{e}^{[\bs\theta_i]_\times}\in \mr{SO}(3)$, and we define a useful shorthand notation $(\mb{R}_i^j)$ to represent the product of rotation matrices that $\mb{R}_i^j=\mb{R}_i\mb{R}_{i+1}\cdots\mb{R}_j$. With the notation, local frame axes in a deformed state (\ding{173}), such as the tangent vector, can be easily computed by $\mb{t}_i=\mb{R}_{0}^{i}\bar{\mb{t}}_i$ from the reference state (\ding{172}). The resulting elastic strain on the $i$th joint will be used to compute the elastic potential energy later. 

The configuration space $(\bs\theta)$ of the SCR can be written as
\begin{equation}
    \label{eqn:c-space}
    \bs\theta=\left[\bs\theta_0, \dots,\bs\theta_{N-1}\right]^\top\in\mathbb{T}\subset\mathbb{R}^{3N}.
\end{equation}
Let a fixed base frame $\{s\}$ attach to joint $0$, frame $\{b\}$ attach to the distal end, and ${}_{b}^{s}\bar{\mb{H}}\in \mr{SE}(3)$ be the associated distal pose when the SCR is in its zero position. 
The screw axes in the fixed base frame are of the form
\begin{equation*}
    \bs\xi_i=\begin{bmatrix}
        \bs\theta_i\\-\bs\theta_i\times\bar{\mb{p}}_i
    \end{bmatrix}\in\mathbb{R}^6
\end{equation*}
where $\bar{\mb{p}}_i$ is the reference position of the $i$th joint associated with $l_i$. 
Then, the forward kinematics of the SCR can be represented in the following product of exponentials:
\begin{equation}
\label{eqn:forwardKine}
    {}_{b}^{s}\mb{H}(\bs\theta)=\mb{e}^{\bs\xi_0^\wedge}\mb{e}^{\bs\xi_{1}^\wedge}\cdots\mb{e}^{\bs\xi_{N-1}^\wedge}{}_{b}^{s}\bar{\mb{H}}(\mb{0})\in \mr{SE}(3).
\end{equation}
We next present the analytical expression of the space Jacobian $\mb{J}_{\bs\theta} (\bs\theta)\in \mathbb{R}^{6 \times 3N}$, defined in fixed frame coordinates. 
The space Jacobian relates small changes in joint space to distal pose as $\delta\mb{h} = \mb{J}_{\bs\theta} \delta\bs\theta$, where $\mb{h}= \mr{Log}\left( {}_{b}^{s}\mb{H} \right)\in \mathbb{R}^6$ represents the distal pose in Cartesian coordinates. 
The Jacobian matrix is partitioned into blocks, with each block corresponding to a joint space group:
\begin{equation}
    \label{eqn:blockJacobian}
    \mb{J}_{\bs\theta}^{s}=\left[
        \frac{\pt\mb{h}}{\pt\bs\theta_0},\ \frac{\pt\mb{h}}{\pt\bs\theta_1},\ \cdots \ ,\frac{\pt\mb{h}}{\pt\bs\theta_{N-1}}
    \right]
\end{equation}
where the derivative with respect to $\bs\theta_i$ is given by
\begin{equation*}
    \frac{\pt\mb{h}}{\pt\bs\theta_i}=\mb{Ad}_{\mb{e}^{\bs\xi_0^\wedge}\cdots \mb{e}^{\bs\xi_{i-1}^\wedge}}\mb{J}_{\mb{H}_l}^{\bs\xi_i}\frac{\pt\bs\xi_i}{\pt\bs\theta_i}\in\mathbb{R}^{6\times 3}
\end{equation*}
with $\frac{\partial\bs\xi_i}{\partial\bs\theta_i}=\begin{bmatrix}
        \mb{I}_3,\  -[\bar{\mb{p}}_i]_\times
\end{bmatrix}^\top$, and $\mb{J}_{\mb{H}_l}^{\bs\xi_i}$ is the left Jacobian of the $\mr{SE}(3)$ manifold as discussed in \cite{barfoot2014associating}. 
To conclude, we note the following relationship between the space Jacobian and the body Jacobian $(\mb{J}_{\bs\theta}^{b}(\bs\theta))$ in case of need:
\begin{equation}
    \label{eqn:bodyJacobian}
    \mb{J}_{\bs\theta}^{b}=\mb{Ad}_{\mb{e}^{\bs\xi_0^\wedge}\cdots \mb{e}^{\bs\xi_{N-1}^\wedge}}^{-1}\mb{J}_{\bs\theta}^{s}.
\end{equation}
Also for brevity, $\mb{J}_{\bs\theta}$ will be referred to $\mb{J}_{\bs\theta}^{b}$ unless specified. 

\subsection{Equations of Motion}

The equations of motion are derived within the framework of Lagrangian mechanics, formulated in terms of the potential energies stored in the SCR. 
The SCR is assumed to undergo bending and twisting deformations and to be driven by generalized actuation sources, which can represent either distributed or localized internal/external torques. 
Accordingly, the total potential energy is composed of elastic strains and actuation-induced moments, decomposed into the elastic $(\mc{E}_e)$ and actuation $(\mc{E}_u)$ energies:
\begin{equation}
\label{eqn:potentialenergy}
\mc{E}(\bs\theta,\bs{u})=\mc{E}_e(\bs\theta)+\mc{E}_u(\bs\theta,\bs{u}),
\end{equation}
where $\bs{u}$ denotes the actuation input. 
For general deformations, the elastic strain at each joint can be approximated as a linear function of the joint angle $\bs\theta_i$~\cite{pittiglio2023closed,roesthuisSteeringMultisegmentContinuum2016}. 
The quadratic form of the elastic strain yields the elastic potential energy
\[
\mc{E}_e = \frac{1}{2}\sum_{i=0}^{N-1}(\bs\theta_i-\bar{\bs\theta}_i)^\top \bs\Lambda_i \bs(\bs\theta_i-\bar{\bs\theta}_i)=\frac{1}{2}(\bs\theta-\bar{\bs\theta})^\top\bs\Lambda(\bs\theta-\bar{\bs\theta}) 
\]
where $\bs\Lambda_i=\operatorname{diag}\left(\tfrac{2G_iI_i}{\ell_i}, \tfrac{E_iI_i}{\ell_i}, \tfrac{E_iI_i}{\ell_i} \right)$, with $G_i$, $E_i$, and $I_i$ denoting the shear modulus, Young’s modulus, and cross-sectional inertia of the $i$-th joint, respectively. The block-diagonal matrix $\bs\Lambda=\operatorname{blkdiag}(\bs\Lambda_0,\cdots,\bs\Lambda_{N-1})$ is referred to as the stiffness matrix. For isotropic material, $G_i=\tfrac{E_i}{2(1+\rho_i)}$, with $\rho_i$ being the Poisson’s ratio. The actuation potential energy $\mc{E}_u(\bs\theta,\bs{u})$ is left in a general form, representing the work done by internal or external actuators. Its gradient with respect to $\bs\theta$ produces the generalized actuation torques $\nabla_{\bs\theta}\mc{E}_u(\bs\theta,\bs{u})$. 

By Lagrangian formulation, neglecting inertial terms in the quasi-static regime, the equilibrium configuration is determined by the equilibrium condition:
\begin{equation}
\label{eqn:lagrangeeq}
\nabla_{\bs\theta}\mc{E}_e(\bs\theta)+\nabla_{\bs\theta}\mc{E}_u(\bs\theta,\bs{u})=\bs\tau(\bs\theta),
\end{equation}
where $\bs\tau(\bs\theta)$ denotes additional generalized torques such as gravity or contact forces.

\subsection{Kinematic Performances} We next introduce the definition of kinematic performance of SCRs with coordinate-invariance. 
We begin by recalling that a Riemannian manifold is a smooth manifold equipped with an inner product on each tangent space that varies smoothly with the base point. 
The associated volume form provides a canonical measure for integration. 
Let $\mathbb{T}\subset\mathbb{R}^n$ be the manifold and use the Euclidean metric, then the induced volume form reduces to the standard Lebesgue measure $\Omega_{\mathbb{T}} = \mrd\bs\theta$. 

Let ${}^{s}_{b}\mb{H}:\mathbb{T}\to SE(3)$ be the forward kinematics of the SCR and $\mb{J}_{\bs\theta}$ its geometric Jacobian. Equipping $SE(3)$ with the Euclidean metric induces the pullback metric \((\mb{J}_{\bs\theta}^{\top}\mb{J}_{\bs\theta})\), whose nonzero eigenvalues $\{\lambda_i(\mb{J}_{\bs\theta}^{\top}\mb{J}_{\bs\theta})\}$ characterize the local distortion of ${}^{s}_{b}\mb{H}$.
Any symmetric function of these eigenvalues defines a spectral functional $z(\lambda_1,\ldots,\lambda_m)$, which serves as a general mathematical form of a kinematic performance measure.
A local performance density can then be defined in general as any symmetric spectral functional \( z({}^{s}_{b}\mb{H}(\bs\theta))\triangleq z(\lambda_1(\bs\theta),\ldots,\lambda_m(\bs\theta)), \) with typical choices including 
\begin{equation}
\label{eqn:kinematicindices}
z({}^{s}_{b}\mb{H}(\bs\theta))=\left\{\begin{aligned} &z_{\mathrm{vol}}(\bs\theta)=\sqrt{\det(\mb{J}_{\bs\theta}\mb{J}_{\bs\theta}^\top)},\\ &z_{\mathrm{dis}}(\bs\theta)=\tfrac{1}{2}\operatorname{tr}(\mb{J}_{\bs\theta}^\top\mb{J}_{\bs\theta}),\\ &z_{\kappa}(\bs\theta)=\tfrac{\sigma_{\max}(\mb{J}_{\bs\theta})}{\sigma_{\min}(\mb{J}_{\bs\theta})}, \end{aligned}\right.     
\end{equation} 
corresponding respectively to volume gain, distortion energy, and condition-number-based dexterity \cite{yoshikawa1985manipulability,parkKinematicDexterityRobotic1994,salisbury1982articulated}. These local measures are extended to global indices by integration with respect to $\Omega_{\mathbb{T}}$ as \(\mc{Z} \triangleq \int_{\mathbb{T}} z({}^{s}_{b}\mb{H}(\bs\theta))\ \Omega_{\mathbb{T}}. 
\)
However, since feasible configurations of the SCR are constrained by the equilibrium condition \eqref{eqn:lagrangeeq}, the effective integration domain is the solution set 
\(\mathbb{T}^*=\{\bs\theta\in\mathbb{T}\mid\nabla_{\bs\theta}\mc{E}(\bs\theta,\bs{u})=\bs\tau(\bs\theta)\}\). 
If the constraint has constant full rank with respect to $\bs\theta$, $\mathbb{T}^*$ is a smooth submanifold of $\mathbb{T}$, and the induced Riemannian volume form can be used.
More generally, including neighborhoods of singularities or branch points, the equilibrium set $\mathbb{T}^*(\bs{u})$ is countably rectifiable with intrinsic dimension $\dim(\mathbb{T}^*(\bs{u}))$, hence it carries a well-defined Hausdorff measure $\mc{H}^{\dim(\mathbb{T}^*)}$ and supports geometric integration. 
Consequently, the constrained global performance can be written as an iterated integral over the actuation parameter space,
\[
\mc{Z}
=\int_{\mathbb{U}}
\bigg(\int_{\mathbb{T}^*}
z\!\left({}^{s}_{b}\mb{H}(\bs\theta)\right)\,
\mrd \mc{H}^{\dim(\mathbb{T}^*)}(\bs\theta)\bigg)\,
\Omega_{\mathbb{U}},
\]
where $\Omega_{\mathbb{U}}$ denotes the volume measure on the admissible actuation domain $\mathbb{U}$.

\section{Structural Optimization Methodology}
\label{sec:III}
\subsection{Problem Formulation}

Let $\mb{a}\in\mathbb{A}$ denote the vector of structural design parameters, such as segment lengths, joint stiffness blocks, or orientations of embedded units. 
The admissible set $\mathbb{A}$ is specified by manufacturing limits and safety considerations. 
For each candidate design $\mb{a}$, the forward kinematics ${}^{s}_{b}\mb{H}(\bs\theta;\mb{a})$ and the corresponding Jacobian $\mb{J}_{\bs\theta}(\bs\theta;\mb{a})$ are uniquely determined, and hence the local performance density $z$ and its global counterpart $\mc{Z}(\mb{a})$ can be evaluated using \eqref{eqn:kinematicindices}.

The structural optimization problem is then formulated as
\begin{equation}
\label{eqn:structopt}
    \begin{aligned}
        \min_{\mb{a}} \ & \mc{Z}(\mb{a})
        \;\triangleq\;
        \int_{\mathbb{T}}
        z({}^{s}_{b}\mb{H}(\bs\theta;\mb{a}))\,\Omega_{\mathbb{T}} \\
        \text{s.t.}\ &
        \begin{cases}
        \nabla_{\bs\theta}\mc{E}(\bs\theta,\bs{u};\mb{a})=\bs\tau(\bs\theta;\mb{a}),
        \quad \bs{u}\in\mathbb{U},\\
        \mc{A}(\mb{a})\leq 0,
        \end{cases}
    \end{aligned}
\end{equation}
where $\mc{A}(\mb{a})\leq 0$ collects algebraic feasibility constraints. Consequently, the global index associated with the structural parameter $(\mb{a})$ can be equivalently written as an iterated integral over the actuation domain and the corresponding solution manifold:
\begin{equation}
    \label{eqn:optprob}
    \begin{aligned}
        \min_{\mb{a}}\mc{Z}(\mb{a})
        &=\int_{\mathbb{U}}
        \Bigg(
        \int_{\mathbb{T}^*}
        z({}^{s}_{b}\mb{H}(\bs\theta;\mb{a}))\,
        \mrd\mc{H}^{\dim(\mathbb{T}^*)}(\bs\theta)
        \Bigg)\,\Omega_{\mathbb{U}}.\\
        \text{s.t.}&\quad\mc{A}(\mb{a})\leq 0.
    \end{aligned}
\end{equation}
In the following, we discuss two complementary strategies to solve the optimization problem: a simplification under smooth manifold assumptions and a gradient-based optimization strategy for general cases.

\subsection{Simplification under Smooth Manifold Assumption}

If the equilibrium constraint admits a smooth parametrization of joint states by the input space, i.e., $\bs\theta^*:\mathbb{U}\to\mathbb{T}$ defines a smooth immersion, then the mapping $\bs\theta^*(\bs{u})$ satisfies the implicit derivative relation:
\begin{equation}
\label{eqn:implicitrelation}
\frac{\partial \bs\theta}{\partial \bs{u}}
=-\mb{S}(\bs\theta,\bs{u};\mb{a})^{-1}\,
\mb{M}(\bs\theta,\bs{u};\mb{a}),
\end{equation}
where 
\begin{equation}
    \label{eqn:immersion}
    \mb{S}=\frac{\pt^2\mc{E}}{\pt\bs\theta\pt\bs\theta}-\frac{\pt\bs\tau}{\pt\bs\theta},\quad \mb{M}=\frac{\pt^2\mc{E}}{\pt\bs\theta\pt\bs{u}}    
\end{equation}
are the generalized stiffness matrix and characterize the coupling with the input, respectively.
The immersion property of $\bs\theta^*$ depends on the rank of $\mb{S}^{-1}\mb{M}$, the mapping is locally smooth and non-degenerate if $\mb{S}$ is nonsingular and $\mb{M}$ has full column rank.
In this case, the feasible set $\mathbb{T}^*$ is diffeomorphic to the input domain $\mathbb{U}$, and the inner integral in \eqref{eqn:structopt} reduces to an evaluation along the immersion
$\bs\theta^*(\bs{u})$. 
By applying the area formula
\cite{federer2014geometric}, one obtains

\begin{equation*}
\begin{aligned}
    \int_{\mathbb{T}^*}
z({}^{s}_{b}\mb{H}(\bs\theta;\mb{a}))\,\mrd\mc{H}^{\dim(\mathbb{T}^*)}(\bs\theta)=
z({}^{s}_{b}\mb{H}(\bs\theta^*;\mb{a}))\,\mc{J}_{\bs\theta}(\bs{u};\mb{a}),
\end{aligned}
\end{equation*}
where $\mc{J}_{\bs\theta}(\bs{u};\mb{a})$ is the Jacobian determinant of the
immersion $\bs\theta^*(\bs{u})$, given by
\[
\mc{J}_{\bs\theta}(\bs{u};\mb{a})
=\left[\det\left(\frac{\partial \bs\theta}{\partial \bs{u}}^\top\frac{\partial \bs\theta}{\partial \bs{u}}\right)\right]^{1/2}\triangleq\left[\mbox{det}\lt\mb{G}_{\bs\theta}\rt\right]^{1/2}
\]
where $\mb{G}_{\bs\theta}$ denotes the Gram matrix of $\frac{\pt\bs\theta}{\pt\bs{u}}$. Substituting back, the global optimization problem \eqref{eqn:structopt} takes the following reformulation:
\begin{equation}
\label{eqn:Zopt}
\begin{aligned}
    \min_{\mb{a}}\mc{Z}(\mb{a})
&=\int_{\mathbb{U}}
z({}^{s}_{b}\mb{H}(\bs\theta^*(\bs{u});\mb{a}))\,
\mc{J}_{\bs\theta}(\bs{u};\mb{a})\,\Omega_{\mathbb{U}}\\
\mr{s.t.}&\quad \mc{A}(\mb{a})\leq 0.
\end{aligned}
\end{equation}
If the structural feasibility condition ($\mc{A}(\mb{a})\leq 0$) represents bound constraints, it can be treated as part of the admissible design domain $\mathbb{A}$. 
Within this domain, the optimization reduces to an unconstrained problem, and the optimality is ensured by satisfying the following first-order necessary condition:
\[
\begin{aligned}
    \nabla_{\mb{a}}\mc{Z}
&=\int_{\mathbb{U}}
\bigg(z\,\nabla_{\mb{a}}\mc{J}_{\bs\theta}
+\mc{J}_{\bs\theta}\,\nabla_{\mb{a}}z\bigg)\,
\Omega_{\mathbb{U}}\\
&=\int_{\mathbb{U}}\left[\frac{1}{2}z\mc{J}_{\bs\theta}\tr\lt \mb{G}_{\bs\theta}^{-1}\nabla_{\mb{a}}\mb{G}_{\bs\theta}\rt+\mc{J}_{\bs\theta}\nabla_{\mb{a}}z\right]\ \Omega_{\mathbb{U}}\\
    &=\mb{0},
\end{aligned}
\]
which, under regularity, yields the strong-form pointwise condition
\begin{equation}
\label{eqn:optimitycondition}
\bigg(\frac{1}{2}z\,\tr(\mb{G}_{\bs\theta}^{-1}\nabla_{\mb{a}}\mb{G}_{\bs\theta})
+\nabla_{\mb{a}}z\bigg)\mc{J}_{\bs\theta}=\mb{0},
\quad \forall \bs{u}\in\mathbb{U},
\end{equation}
where $\mb{G}_{\bs\theta}$ is the Gram matrix of $\tfrac{\partial\bs\theta}{\partial\bs{u}}$.
Degeneracy occurs when $\mc{J}_{\bs\theta}=0$, corresponding to collapsed input directions where actuation can not produce independent generalized torques.
If $z$ is a non-zero kinematic performance index, one can further deduce the non-trivial solution of \eqref{eqn:optimitycondition} that
\begin{equation}
\label{eqn:gradient}
\tr(\mb{G}_{\bs\theta}^{-1}\nabla_{\mb{a}}\mb{G}_{\bs\theta})=\sum_{i}\frac{\nabla_{\mb{a}}\lambda_i(\mb{G}_{\bs\theta})}{\lambda_i(\mb{G}_{\bs\theta})}=-2\frac{\nabla_{\mb{a}}z}{z}.
\end{equation}
Geometrically, this condition implies that optimizing the local performance requires structurally modulating the spectrum of the configuration space metric to counteract its distortion.


\subsection{Gradient-based Optimization for General Cases}

The smooth manifold assumption may fail in practice when the generalized stiffness matrix $(\mb{S})$ becomes singular or ill-conditioned. 
In such cases, the implicit relation \eqref{eqn:implicitrelation} is nolonger to hold, and the equilibrium mapping $\bs\theta^*$ cannot be obtained analytically.
To address this difficulty, we reformulate the optimization as a nested energy-minimization problem that unifies both smooth and singular cases within a single computational framework.

\subsubsection{Inner-level equilibrium computation}

For each input $\bs{u}$ and structural parameter $\mb{a}$, the equilibrium configuration is computed as the stationary point of the total potential energy:
\begin{equation*}
    \bs\theta^*(\bs{u};\mb{a})=\arg\min_{\bs\theta}\lt\mc{E}(\bs\theta,\bs{u};\mb{a})-\int\bs\tau\mrd\bs\theta\rt.
\end{equation*}
The minimization is solved iteratively using a trust-region scheme 
\begin{equation}
    \label{eqn:trust-region-update}(\mb{S}+\mu\mb{I})\Delta\bs\theta=\bs\tau(\bs\theta)-\nabla_{\bs\theta}\mc{E},\quad \bs\theta\leftarrow\bs\theta+\Delta\bs\theta.
\end{equation}
The damping factor $\mu>0$ is adaptively updated based on the ratio of the actual to the predicted energy reductions.
During the actuation, the solution for each $\bs{u}$ is warm-started from the previous equilibrium, enabling smooth continuation across potential branch points.

\subsubsection{Outer-level structural update}
Once the inner equilibrium mapping is obtained, the global performance index is evaluated numerically by quadrature.
The structural parameters are then updated following a projected gradient-descent rule: 
\[
\mb{a}_{k+1}=\mc{P}_{\mathbb{A}_c}(\mb{a}_{k}-\eta\nabla_{\mb{a}}\mc{Z}(\mb{a}_k)),\quad\mathbb{A}_c=\{\mb{a}\vert\mc{A}(\mb{a})\leq0,\mb{a}\in\mathbb{A}\}.
\]
where $\mc{P}_{\mathbb{A}_c}(\cdot)$ denotes the projection onto the admissible design domain and $\eta>0$ is a step size ensuring monotonic decrease of $\mc{Z}$.
The gradient $(\nabla_{\mb{a}}\mc{Z})$ can be approximated by finite differences or obtained via automatic differentiation through the inner minimization loop. Consequently, the complete computational process is summarized in Algorithm \ref{alg:two_level_opt}.

\begin{algorithm}[t]
{\small
\caption{Two-level Structural Optimization Framework}
\label{alg:two_level_opt}
\begin{algorithmic}[1]
\Require{Initial structural parameters $\mb{a}_0$, admissible domain $\mathbb{A}_c$, actuation set $\mathbb{U}$,
initial configurations $\bs\theta_0(\bs{u})$, damping factor $\mu_0>0$.}
\Ensure{Optimized structural parameters $\mb{a}^*$.}
\For{$k=0,1,\ldots$ until convergence}
    \For{each actuation input $\bs{u}_j\in\mathbb{U}$}
        \State Initialize $\bs\theta\leftarrow\bs\theta_0(\bs{u}_j)$, $\lambda\leftarrow\lambda_0$.
        \Repeat
            \State Compute gradient $\nabla_{\bs\theta}\mc{E}(\bs\theta,\bs{u}_j;\mb{a}_k)$
            \State Compute the stiffness matrix $\mb{S}$.
            \State Solve and update the trust-region step \eqref{eqn:trust-region-update}.
        \Until{convergence of $\mc{E}$}
        \State Store equilibrium $\bs\theta^{*}(\bs{u}_j;\mb{a}_k)\leftarrow\bs\theta$.
    \EndFor
    \State Evaluate the global performance index $\mc{Z}(\mb{a}_k)$.
    \State Update structural parameters by projected gradient descent:
           \(
           \mb{a}_{k+1}=
           \mc{P}_{\mc{A}_c}\!\big(\mb{a}_k-\eta\,\nabla_{\mb{a}}\mc{Z}(\mb{a}_k)\big).
           \)
\EndFor
\State \Return $\mb{a}^*\leftarrow\mb{a}_k$.
\end{algorithmic}
}
\end{algorithm}


\section{Application of the Optimization Framework to Magnetic Actuation}
\label{sec:IV}
To demonstrate the generality and practical relevance of the proposed optimization framework, it is applied here to MeSCRs, whose actuation arises from the interaction between internal permanent magnets and external magnetic fields.

\subsection{Magnetic Actuation}
We formulate the magnetic potential energy and its differential properties with respect to the robot configuration.
Each embedded magnet is treated as an independent joint, with additional joints added on either side to assist the description. The joint indices of all $N_m$ micromagnets constitute an index set $\mathbb{K}\subset\mathbb{Z}$.
Given that the robot operates under a constant temperature, the formulation aligns with the standard Helmholtz free energy without thermal variations \cite{linMagneticContinuumRobot2021}.
In this case, the magnetic potential energy of the MeSCR is given by
\begin{equation*}
    \mc{E}_m = -\sum_{k\in\mathbb{K}} \mb{m}_k^\top\mb{b}(\mb{p}_k)
\end{equation*}
where $\mb{m}_k=\mb{R}_{0}^{k}\bar{\mb{m}}_k\in\mathbb{R}^3$ is the magnetic dipole moments of the $k$th magnet with modulus $M_k=\|\bar{\mb{m}}_k\|$ and $\mb{b}(\mb{p}_k)\in\mathbb{R}^3$ denote the external magnetic field with modulus $B_k=\|\mb{b}(\mb{p}_k)\|$. For readibility, we sometimes denote $\mb{b}(\mb{p}_k)$ in $\mb{b}_{k_i}$, in which $i=0,1,\dots,N_m-1$.

The gradient and Hessian of $\mc{E}_e$ are straightforward, whereas we focus on computing that of $\mc{E}_m$. 

\begin{equation*}
\begin{aligned}
    \nabla_{\bs\theta_i}\mc{E}_m
    =&-\sum_{k\in\mathbb{K},k\geq i}\lt\left[\mb{m}_{k}\right]_\times^\top\mb{R}_{0}^{i}\mb{J}_{\mb{R}_r}^{\bs\theta_i}+\frac{\pt\mb{p}_k}{\pt\bs\theta_i}(\mb{m}_k^\top\nabla_{\mb{p}_i})\rt^\top\mb{b}_k\\
    \triangleq&\mb{M}_i(\bs\theta)\mb{b}
\end{aligned}
\end{equation*}
where $\mb{J}_{\mb{R}_r}^{\bs\theta_i}$ represents the right Jacobian of $\mr{SO}(3)$ manifold associated with $\bs\theta_i$ \cite{chirikjianStochasticModelsInformation2011}. We reorganize the expression by grouping the multiplicative terms into matrix $\mb{M}_i(\bs\theta)\in\mathbb{R}^{3\times 3N_m}$ and introducing vector $\mb{b}=\left[\mb{b}_{k_0}^\top,\dots,\mb{b}_{k_{N_m-1}}^\top\right]^\top\in\mathbb{B}\subset\mathbb{R}^{3N_m}$ to obtain a compact form. 
For precise calculations, $\frac{\pt\mb{p}_k}{\pt\bs\theta_i}$ can be derived from \eqref{eqn:bodyJacobian}.
As a result, the equilibrium condition \eqref{eqn:lagrangeeq} has the following explicit representation:
\begin{equation}
    \label{eqn:e_LagrangianEquation}
    \bs\Lambda(\bs\theta-\bar{\bs\theta})+\bs\tau(\bs\theta)=\mb{M}(\bs\theta)\mb{b}
\end{equation}
where $\mb{M}(\bs\theta)=\left[\mb{M}_0^\top,\dots,\mb{M}_{N-1}^\top\right]^\top\in\mathbb{R}^{3N\times 3N_m}$. 
By developing \eqref{eqn:e_LagrangianEquation} and performing some rearrangements, the joint space compliance relation between the variation of the actuating magnetic field $(\delta\mb{b})$ and the configuration variation $(\delta\bs\theta)$ is obtained:
\begin{equation*}
    \lt\bs\Lambda+\frac{\pt^2\mc{E}_m}{\pt\bs\theta\pt\bs\theta^\top}+\frac{\pt\bs\tau}{\pt\bs\theta}\rt\delta\bs\theta\triangleq \mb{S}\delta\bs\theta=\mb{M}(\bs\theta)\delta\mb{b}
\end{equation*}
where $\mb{S}\in\mathbb{R}^{3N\times 3N}$ is known as the Hessian matrix of stiffness.  
Observed that, the magnetic actuation of the MeSCR follows the canonical equilibrium structure of \eqref{eqn:immersion}, where variations in the magnetic field $\mb{b}$ correspond to the control inputs $\bs{u}$ that drive configuration changes.


\subsection{Properties of Magnetic Actuation under Uniform Fields}
We further analyze the properties of the actuation mapping in a uniform magnetic field, where the field can be expressed as \(\mb{b}=\mb{U}_u\mb{b}_u\) with $\mb{U}_u=\mb{1}_{N_m}\otimes\mb{I}_3$ and $\mb{b}_u\in\mathbb{R}^3$.
We begin by establishing two fundamental properties of the magnetic actuation mapping:
\begin{proposition}
    \label{pro:1}
    The matrix-valued function $\mb{M}(\bs\theta)$ and the magnetic torque $\mb{M}(\bs\theta)\mb{b}$ are bounded on $\mathbb{R}^{3N}$ with positive constants $\mc{M}_0$ and $\mc{M}$ such that $\|\mb{M}(\bs\theta)\|\leq\mc{M}_0$ and $\|\mb{M}(\bs\theta)\mb{b}\|\leq\mc{M}$ for all $\bs\theta\in\mathbb{R}^{3N}$, respectively.
\end{proposition}
\begin{proposition}
    \label{pro:2}
    The vector-valued function $\mb{M}(\bs\theta)\mb{b}$ is Lipschitz continuous on $\mathbb{R}^{3N}$ with a Lipschitz constant $\mc{L}$ such that
    $\|\mb{M}(\bs\theta)\mb{b}-\mb{M}(\bs\varphi)\mb{b}\|\leq\mc{L}\|\bs\theta-\bs\varphi\|$ for all $\bs\theta,\bs\varphi\in\mathbb{R}^{3N}$.
\end{proposition}
\begin{proof}
    The proof is deferred to supplementary materials.
\end{proof}
Assuming the robot to be naturally straight $(\bar{\bs\theta}=\mb0)$, and $\bs{\tau}(\bs\theta)$ is ruled out, \eqref{eqn:e_LagrangianEquation} further simplifies to the following fixed-point equation:
\begin{equation}
    \label{eqn:stdcontract}
    \bs\theta=\bs\Lambda^{-1}\mb{M}(\bs\theta)\mb{b}
\end{equation}
Under this formulation, the existence and uniqueness of the equilibrium solution can be guaranteed by ensuring that the mapping is a contraction in the configuration space:
\begin{theorem}
\label{thm:solunique} 
There exists a unique equilibrium solution to \eqref{eqn:stdcontract} if the following inequality holds:
    \begin{equation}
    \label{eqn:thm1cond}
    \max_{k\in\mathbb{K}} B_k<\frac{\pi}{4}\frac{\lambda_{\min}\lt\bs\Lambda\rt}{\sum_{k\in\mathbb{K}}kM_k}
\end{equation}
where $\lambda_{\min}(\cdot)$ denotes the smallest eigenvalue of a matrix.
\end{theorem}
\begin{proof}
\label{prf:1}
Based on the Proposition \ref{pro:2}, the Lipschitz constant of the continuous mapping can be estimated by
\begin{equation*}
    \left\|\bs\Lambda^{-1}\frac{\pt}{\pt\bs\theta}\mb{M}(\bs\theta)\mb{b}\right\|=\|\bs\Lambda^{-1}\mb{S}_m\|\leq\|\bs\Lambda^{-1}\|\mc{L}=\lambda_{\min}^{-1}(\bs\Lambda)\mc{L}.
\end{equation*}
From \eqref{eqn:thm1cond}, it follows that $\lambda_{\min}^{-1}(\bs\Lambda)\mc{L}<1$. Then, $\bs\Lambda^{-1}\mb{M}(\bs\theta)\mb{b}$ is a contraction mapping on $\mathbb{R}^{3N}$. By Banach’s fixed-point theorem, there exists a unique equilibrium point to \eqref{eqn:stdcontract}.
\end{proof}
The inequality provides a sufficient condition to guarantee the uniqueness of a globally attractive equilibrium point for any initial condition. 
Notably, the upper bound is solely determined by the material properties and is independent of the joint number \(N\) (note that $N\bar{\ell}_i\approx L$). 
The limitation on the external magnetic field strength can be explained by the critical buckling instability of the magneto-elastic material \cite{lu2023mechanics}. We found it intriguing that when each joint is magnetized, the resulting upper bound exhibits a striking structural resemblance to that derived in previous work \cite{Wu2024Closed} on magnetic continua, differing only slightly in the coefficients. 
Furthermore, Theorem \ref{thm:solunique} also ensures \(\mathbf{S}\) to be invertible. By applying Weyl's inequality, one can easily verify that \(\lambda_{\min}(\mathbf{S}) \geq \lambda_{\min}(\boldsymbol{\Lambda}) - \mathcal{L} > 0\).

We subsequently define the controllable degrees of freedom (DoF) as the independently actuated direction of the task-space twist, which corresponds to the row rank of the actuation Jacobian $(\mb{J}_\mb{b}=\mb{J}_{\bs\theta}\frac{\pt\bs\theta}{\pt\mb{b}})$. 
The following theorem reveals the relation between it and the configuration of embedded magnets:
\begin{theorem}
    \label{thm:DoF}
    The controllable DoF of the MeSCR does not exceed twice the number of embedded magnets.
\end{theorem}
\begin{proof}
    It is deferred to the supplementary materials.
\end{proof}

The controllable DoF is ultimately constrained by the number and independence of the actuating fields. 
As $N_m$ increases, achieving fully actuated or redundant control would require an equal number of independently controllable fields, which is generally infeasible in practice. 
Under the spatially uniform condition on fields, the maximum controllable DoF does not exceed three. 
For simplicity and manufacturability, all embedded magnets are assumed to share identical material and geometric properties.

In most commercial designs, embedded magnets are placed along the catheter’s axial direction. 
For $N_m$ magnets, this yields $2^{N_m}$ orientation permutations.
However, because $[\pm\mathbf{m}_k]_\times\mathbf{b}_k = [\mathbf{m}_k]_\times(\pm\mathbf{b}_k)$, flipping a magnet’s polarity can be equivalently compensated by reversing the field direction. 
Thus, axial alignment effectively reduces the optimization space while restricting the MeSCR to a twist-free deformation mode due to the isotropic and axially symmetric nature of its soft matrix.

\begin{theorem}[Material-twist-free]
\label{thm:twistfree}
The MeSCR is material-twist-free ($\bs\theta_i^\top\bar{\mb{t}}_i=0$) for any actuating magnetic field, provided that the embedded magnets' magnetic moments are aligned axially, i.e., \(\bar{\mb{m}}_k=\pm M_k\bar{\mb{t}}\), and the magnetic field strength satisfies:
$\max_{k\in\mathbb{K}} B_k \leq \sqrt6 \lambda_{\min}(\bs\Lambda)/\sum_{k\in\mathbb{K}} M_k.$
\end{theorem}
\begin{proof}
See supplementary material for details.
\end{proof}

Noticing that \(\frac{\sqrt6\lambda_{\min}(\bs\Lambda)}{\sum_{k\in\mathbb{K}}M_k} \propto N^2\), as \(N \to \infty\), Theorem \ref{thm:twistfree} holds for any actuating magnetic field. This upper bound characterizes the ability of the discrete model with \(N\) joints to approximate the continuum. 
We found an alternative form of the equilibrium equation \eqref{eqn:stdcontract} from the underlying proof:
\begin{equation}
    \label{eqn:newiteration}
    \bs\theta_i^*=\frac{\sum_{k\geq i}\mb{v}_k}{\lambda_{\max}(\bs\Lambda_i)}=\frac{\sum_{k\geq i}[\mb{R}_{0}^{k^\top}(\bs\theta^*)\mb{b}_k]_\times^\top\bar{\mb{m}}_k}{\lambda_{\max}(\bs\Lambda_i)}.
\end{equation}
The configuration space \eqref{eqn:c-space} thus can be degenerated to a simplified model as presented in \cite{greigarnPseudorigidbodyModelKinematic2015}. Moreover, the joint stiffness matrix reduces to $\bs\Lambda_i=\Lambda_i\mb{I}_3$ with $\Lambda_i=\frac{E_iI_i}{\ell_i}$, and we denote $\lambda_{\min}(\bs\Lambda)=\Lambda_{\min}$. We emphasize its distinction from geometric-twist-free, which requires $\bs\theta_i^\top\mb{t}_i=0$. It should be known that geometric torsion may still exist under the material-twist-free condition. 
However, if the external magnetic field is coplanar, the following corollary holds:
\begin{corollary}
    \label{cor:1}
    Define the plane $\Pi=\operatorname{span}\{\bar{\mb{t}},\mb{s}\}\subset\mathbb{R}^3$, where $\mb{s}\in\mathbb{R}^3$ is linearly independent of $\bar{\mb{t}}$. Suppose $\forall k\in\mathbb{K}$, $\bar{\mb{m}}_k=\pm M_k\bar{\mb{t}}$ and $\mb{b}_k\in\Pi$. Then there exists an equilibrium configuration $\bs\theta^*\in\mathbb{R}^{3N}$ such that the centerline of the MeSCR lies entirely within the plane $\Pi$.
\end{corollary}

\begin{proof}
    The proof is constructive. Let $\bs\theta_i^*=\theta_i^*\mb{n}$, where $\mb{n}=\frac{\bar{\mb{t}}\times\mb{s}}{\|\bar{\mb{t}}\times\mb{s}\|}$ is the unit vector orthogonal to $\Pi$. Since $\mb{b}_k\in\Pi$, we notice that \(\mb{R}_{0}^{k^\top}(\bs\theta^*)\mb{b}_k=\mb{R}(-\sum_i\theta_i^*\mb{n})\mb{b}_k\in\Pi\). Then $\forall k\in\mathbb{K}$, $[\bar{\mb{m}}_k]_\times\mb{R}_0^{k^\top}\mb{b}_k\!\in\operatorname{span}\{\mb{n}\}$. Hence $\bs\theta_i^*=\theta_i^*\mb{n}$ is a solution to $\eqref{eqn:newiteration}$. The corresponding tangent vectors  $\mb{t}_i=\mb{R}_0^i\bar{\mb{t}}_i\in\Pi$, thus the centerline lies entierly within the plane $\Pi$.
\end{proof}
Specially, if the initial value is chosen as $\bs\theta(0)=\mb{0}$, then the gradient-based iteration evolves within the subspace $\operatorname{span}\{\mb{n}\}$. 
The optimization trajectory converges to a planar equilibrium configuration satisfying $\bs\theta_i=\theta_i\mb{n}$. 
We can further derive a closed-form expression under the small-angle premise, wherein the cumulative rotation matrix admits the first-order expansion:
$
\mb{R}\lt\sum_{i=0}^k\theta_i\mb{n}\rt\approx\mb{I}_3+\sum_{i=0}^k\theta_i[\mb{n}]_\times
$. Substituting this into the total potential energy yields
\begin{equation*}
    \mc{E}\approx\sum_{i=0}^{N-1}\frac12\Lambda_i\theta_i^2+\sum_{i=0}^{N-1}\theta_i\!\lt\sum_{k\in\mathbb{K},k\geq i}\bar{\mb{m}}_k^\top[\mb{n}]_\times\mb{b}_k\rt-\sum_k\bar{\mb{m}}_k^\top\mb{b}_k.
\end{equation*}
The optimization problem $\min_{\bs\theta}\mc{E}$ thus has the explicit solution
\begin{equation}
    \label{eqn:explicitsolution}
    \theta_i=\frac{1}{\Lambda_i}\sum_{k\in\mathbb{K},k\geq i}\bar{\mb{m}}_k^\top[\mb{n}]_{\times}^\top\mb{b}_k.
\end{equation}
It can be verified that this also corresponds to the general solution of \eqref{eqn:newiteration} under the small-angle premise.

\subsection{Analytical Optimization under Weak Uniform Fields}

If the external magnetic field remains unchanged, the deformation of MeSCRs is determined by the placement (position) and orientation of embedded magnets. 
For instance, embedded magnets with variable positions have been employed to expand the MeSCR's workspace \cite{parkWorkspaceExpansionMagnetic2024}. 
In the PRB model, the length of embedded magnets is defined by the Voronoi length as $L_k\triangleq\sum_{i\leq k}\ell_i, \forall k\in\mathbb{K}$. 
From \eqref{eqn:stdcontract}, it follows that the flexible segment beyond the last embedded magnet is unactuated. 
It does not affect the workspace volume of the MeSCR, nor any performance metric defined using a left-invariant Riemannian metric. 
Therefore, we let one of the magnets be fixed at the distal end $(L_{k_{N_m-1}}=L)$, leaving $N_m-1$ decision variables on placements. 
Additionally, we let the flexible segments between adjacent magnets have minimal gaps $(L_g)$ to mitigate magnetic coupling.
The orientation of embedded magnets $(\bar{\mb{m}}_k\in\mathbb{S}^2)$ is also considered as the design variable.

We collect the design variables of magnet placements and orientations into the parameter vector $\mb{a}\in\mathbb{A}$. 
Under a spatially uniform magnetic field $\mb{b}\in\mathbb{B}$ serving as the actuation source, the structural optimization of the MeSCR follows the general form of~\eqref{eqn:optprob}, where the equilibrium mapping $\bs\theta^\ast(\mb{b};\mb{a})$ is implicitly defined by the fixed-point relation in~\eqref{eqn:stdcontract}. 
Suppose the magnetic field $\mb{b}$ satisfies the condition in Theorem~\ref{thm:solunique}. 
The manifold $\mathbb{B}$ then constitutes an affine subspace diffeomorphic to $\mathbb{R}^3$. 
By Theorem~\ref{thm:solunique}, the stiffness matrix $\mb{S}$ is nonsingular in a neighborhood of each $\mb{b}_u\in\mathbb{B}$, ensuring that the equilibrium mapping is differentiable with respect to $\mb{b}$. 
Furthermore, by Theorem~\ref{thm:DoF}, the Jacobian $\frac{\partial\bs\theta}{\partial\mb{b}}=\mb{S}^{-1}\mb{M}$ satisfies $\operatorname{rk}(\mb{S}^{-1}\mb{M})=2N_m$, implying that $\frac{\partial\bs\theta}{\partial\mb{b}_u}=\mb{S}^{-1}\mb{M}\mb{U}_u$ is of full rank when $N_m\ge2$. 
Consequently, the immersion $\bs\theta^\ast:\mathbb{B}\to\mathbb{T}$ is a smooth mapping, and the structural optimization of the MeSCR can be expressed in the analytical integral form of~\eqref{eqn:Zopt} derived under the smooth manifold assumption.

Since $\mc{J}_{\bs\theta}$ contains complex Hessian terms, we first derive a tractable approximation to facilitate further analysis.
\begin{lemma}
\label{lmm:approximmersionJacobian}
Suppose the inequality condition in Theorem~\ref{thm:solunique} holds. Then the immersion Jacobian admits the expansion
\begin{equation}
\label{eqn:expansionofJ}
\mc{J}_{\bs\theta}^2 = \det\left(\mb{U}_u^\top\mb{M}^{\top} \bs\Lambda^{-2}\mb{M}\mb{U}_u \right)+\epsilon\triangleq\tilde{\mc{J}}_{\bs\theta}^2 + \epsilon
\end{equation}
where the remainder term $\epsilon$ is bounded above as
\[
|\epsilon| \leq \frac{3(2\mc{L} - \mc{L}^2)}{(1 - \mc{L})^2} \frac{N_m^3 \mc{M}_0^6}{\Lambda_{\min}^6},
\]
and satisfies $\lim_{B_k \to 0} |\epsilon| \to 0$.
\end{lemma}
\begin{proof}
It is deferred to the supplementary material.
\end{proof}
It shows that $\tilde{\mc{J}}_{\bs\theta}$ provides a close approximation of $\mc{J}_{\bs\theta}$ in the weak-field regime, with the residual term converging to zero as the magnetic field strength decreases. 
We can now turn to analyze the approximation of the objective $\tilde{\mc{Z}}(\mb{a}) \triangleq\int_{\mathbb{B}}z\  \tilde{\mc{J}}_{\bs\theta}\ \Omega_{\mathbb{B}}$. 
Let $\mb{M}(\bs\theta)$ in (S15) be represented as $\mb{M}=\mb{D}_{k_m-1}^\top\mb{U}_m\mb{D}_m$, we notice
\begin{equation*}
\begin{aligned}
    \mb{U}_m^\top\mb{D}_{k_{N_m-1}}\bs\Lambda^{-2}\mb{D}_{k_{N_m-1}}^\top\mb{U}_m&=\mb{U}_m^\top\bs\Lambda^{-2}\mb{U}_m=\bs\Lambda_m\!\otimes\!\mb{I}_3
\end{aligned}
\end{equation*}
where
\begin{equation*}
    \boldsymbol\Lambda_m=\begin{bmatrix}
\Lambda^m_{k_0},&\Lambda^m_{k_0},&\dots&\Lambda^m_{k_0}\\
\Lambda^m_{k_0},&\Lambda^m_{k_1},&\dots&\Lambda^m_{k_1}\\
\vdots&\vdots&\ddots&\vdots\\
\Lambda^m_{k_0},&\Lambda^m_{k_1},&\dots&\Lambda^m_{k_{N_m-1}}
\end{bmatrix}\in\mathbb{R}^{N_m\times N_m}
\end{equation*}
with
$\Lambda_{k_i}^m=\sum_{j=0}^{k_i}\frac{\ell_j^2}{E_j^2I_j^2}\approx\frac{1}{k_0E^2I^2}\sum_{j=0}^{i}(L_{k_j} - L_{k_{j-1}})^2$ as a function of $L_k$. 
Lemma~\ref{lmm:approximmersionJacobian} also derives the approximation:
\begin{equation}
\label{eqn:approxiGram}
\begin{aligned}
    \mb{G}_{\bs\theta}&\approx\mb{U}_u^\top\mb{D}_m^\top\lt\bs\Lambda_m\otimes\mb{I}_3\rt\mb{D}_m\mb{U}_u\\
    &=\sum_{ij}\Lambda_{ij}^m\lt\mb{m}_{k_i}\mb{m}_{k_j}^\top-\lt\mb{m}_{k_i}^\top\mb{m}_{k_j}\rt\mb{I}_3\rt.
\end{aligned}
\end{equation}
The following lemma shows the relation between the singularity of $\tilde{\mc{J}}_{\bs\theta}$ and the configuration of $L_k$:
\begin{lemma}
\label{lmm:singularity}
Suppose the embedded magnets are placed at the two boundaries (proximal and distal ends), that is $\forall 0\leq N_k<N_m-1$, $L_{k_{N_k}} = 0$ and $L_{k_{N_k+1}} = L$. Then $\tilde{\mc{J}}_{\bs\theta} = 0$.
\end{lemma}
\begin{proof}
    It is deferred to the supplementary material.
\end{proof}



Lemma~\ref{lmm:singularity} indicates that placing all embedded magnets at boundary positions leads to a singularity in $\tilde{\mc{J}}_{\bs\theta}$, as expected. This degeneration of the measure can also be interpreted in terms of an effective reduction in the number of independent magnets. Specifically, magnets located at the proximal end have no actuation effect, while those at the distal end act collectively as a single unit. As a result, under boundary-concentrated configurations, the effective number of embedded magnets is reduced to one. From Theorem~\ref{thm:DoF}, we notice $\operatorname{rk}\lt\mb{M}(\bs\theta)\rt=2$ when $N_m=1$, and the Jacobian $\frac{\pt\bs\theta}{\pt\mb{b}_u}$ becomes deficient in rank, leading to a singularity of immersion Jacobian. This structural degeneracy provides a tractable mechanism for determining the optimal solution to a class of objective functions, as formalized in the following theorem.
\begin{theorem}
\label{thm:optimalresult}
    If the performance index satisfies $z(\mb{b}_u;L_k)>0$, and the structural parameter $L_k$ satisfies the boundary configuration condition in Lemma~\ref{lmm:singularity}, then the objective $\tilde{\mc{Z}}$ attains its minimum at this configuration.
\end{theorem}

\begin{proof}
Since \(z > 0\) and \(\tilde{\mc{J}}_{\bs\theta} \geq 0\), it follows that \(\tilde{\mathcal{Z}}\geq 0\) for any admissible \(L_k\). When \(L_k\) satisfies the boundary configuration condition in Lemma~4, we have \(\tilde{\mc{J}}_{\bs\theta}(\mb{b}) \equiv 0\) over \(\mathbb{B}\), and thus \(\tilde{\mathcal{Z}} = 0\). 
\end{proof}
Due to the Jacobian singularity at boundary configurations, the value of the objective becomes insensitive to the gradient of the performance index. However, the minimum spacing constraint prevents the structural parameter $L_k$ from reaching such degenerate configurations. Since $\mc{J}_{\bs\theta}$ remains continuous on $L_k$, it stays small in the vicinity of the boundary, making nearby configurations highly competitive in minimizing the objective. In this case, the performance gradient becomes the key factor in selecting specific configurations. One can also reduce the optimization bias by considering the objective as $\mc{Z}_{\mr{vol}}/\mc{Z}_{\mr{dex}}$. This normalized objective reflects the average performance on the configuration space and is insensitive to Jacobian singularity.

\textbf{Example: }
Consider a MeSCR with two axially embedded magnets actuated by a spatially uniform field $\mb{b}_u$ with $\mb{s}^\top\bar{\mb{t}}=0$, where the global manipulability is to be maximized \cite{yoshikawa1985manipulability}. 
The embedded magnets have identical dipole strength. 
Let one magnet be fixed at the distal end (\(L_{k_1} = L\)), while the position of the other magnet \(L_{k_0}\) serves as the design variable, subject to the boundary constraint \(0 < L_{k_0} < L\). 
For planar deflection, 
the joint variable simplifies to \(\boldsymbol{\theta}_i = \theta_i\mb{n}\), with $\mb{n}\in\mathbb{S}^2$ represents the normal vector.
In this configuration, the MeSCR becomes kinematically equivalent to a planar revolute open-chain mechanism. 
Let \(\theta_i^j \triangleq \sum_{l=i}^j \theta_l\) denote the cumulative joint angle. 
In the regime of weak magnetic actuation, the global manipulability index has the linearized expression on $\theta_{k_0}^{k_1}$: \(z\approx L_{k_0}(L - L_{k_0}) \left|\sin\theta_{k_0}^{k_1} \right|
\), where the analytical form is derived from the planar 2-joint revolute open-chain mechanism.

\subsubsection{Optimal Orientation}
We parameterize the in-plane magnetic moments by angles $\{\varphi_k\}$ with respect to the $\bar{\mb{v}}$-axis, where the initial orientation is $\bar{\mb{t}}$.
Let $\varphi_{k_0}^{k_1}\!\triangleq\!\varphi_{k_1}-\varphi_{k_0}$ denote the relative orientation, and assume a spatially uniform planar field $\mb{b}_u\!\in\!\mathbb{R}^3$ with support on the $\bar{\mb{t}}$-$\bar{\mb{u}}$ plane. 
Under the weak-field regime and planar bending, the immersion Jacobian admits the linearized form
\(
\mc{J}_{\bs\theta}\;\approx\;c_0\big|\sin\!\big(\varphi_{k_0}^{k_1}+\theta_{k_0}^{k_1}\big)\big|,
\label{eq:ori_Jlin}
\)
where $c_0>0$ collects geometry- and stiffness-related constants.
The corresponding global performance functional can be written as
\begin{equation}
    \mc{Z}\!\left(\varphi_{k_0}^{k_1}\right)
    \;=\;
    c_1\!(
        \sin(\varphi_{k_0}^{k_1}+2\theta_{k_0}^{k_1})/4
        \;-\;
        \theta_{k_0}^{k_1}\cos(\varphi_{k_0}^{k_1}
    )/2),
    \label{eq:ori_objective}
\end{equation}
with $c_1>0$ independent of $\varphi_{k_0}^{k_1}$. 
In the small-angle limit, the stationarity condition leads to
\begin{equation}
    \lim_{\theta_{k_0}^{k_1}\to0}\nabla_{\varphi_{k_0}^{k_1}}\mc{Z}=0
    \;\;\Longrightarrow\;\;
    \big(\varphi_{k_0}^{k_1}\big)^\ast \;=\; \frac{\pi}{2}\;(\text{mod }\pi),
    \label{eq:ori_opt_small}
\end{equation}
indicating that the optimal in-plane orientation is orthogonal in the uniform weak field regime.
Interestingly, it is analogous to the orthogonal joint arrangement in serial manipulators, where perpendicular joint axes maximize manipulability by decoupling motion directions around the nominal configuration.
Similar orthogonal magnetization has been widely adopted in radially magnetized catheter robots~\cite{leeSteeringTunnelingStent2021,zhangMagneticallyActuatedMicrocatheter2025}, where the magnetic moment is aligned perpendicular to the backbone to enhance steering torque and directional responsiveness under uniform magnetic fields.


\subsubsection{Optimal Placement}
We focus on axially magnetized configurations, as they are the most common design adopted in continuum magnetic robots.
Applying the approximation from \eqref{eqn:approxiGram}, the immersion Jacobian also admits the linearized form:
\begin{equation*}
    \begin{aligned}
        &\mc{J}_{\bs\theta} \approx\\
    &\begin{cases}
        \!M^3\!\sqrt{\Lambda_{k_0}^m (\Lambda_{k_1}^m - \Lambda_{k_0}^m)(3\Lambda_{k_0}^m + \Lambda_{k_1}^m)} \left| \theta_{k_0}^{k_1}\!\right|,\text{if } \bar{\mathbf{m}}_{k_0}\!=\!\bar{\mathbf{m}}_{k_1}, \\
        \!M^3\!\sqrt{\Lambda_{k_0}^m} (\Lambda_{k_1}^m - \Lambda_{k_0}^m)\left| \theta_{k_0}^{k_1}\!\right|,\text{if } \bar{\mathbf{m}}_{k_0} = -\bar{\mathbf{m}}_{k_1}, 
    \end{cases}
    \end{aligned}
\end{equation*}
The cumulative rotation angle can also be approximated using \eqref{eqn:explicitsolution} as $\left| \theta_{k_0}^{k_1} \right| \approx\frac{L - L_{k_0}}{EI}\left|\bar{\mb{m}}_{k_1}^\top[\mb{n}]_\times^\top\mb{b}_u\right|$. Substituting this into the above expressions for \(z\) and \(\mc{J}_{\bs\theta}\), we can factor the dependence on $L_{k_0}$ out of the integral in the objective. The structural optimization problem thus reduces to:
\begin{equation*}
\begin{aligned}
    \max_{L_{k_0}}\begin{cases}
        \!L_{k_0}^2(L-L_{k_0})^4\sqrt{4L_{k_0}^2\!+\!(L-L_{k_0})^2},\mbox{if}\ \bar{\mb{m}}_{k_0}\!=\!\bar{\mb{m}}_{k_1},\\
        \!L_{k_0}^2(L-L_{k_0})^5,\mbox{if}\ \bar{\mb{m}}_{k_0}\!=\!-\bar{\mb{m}}_{k_1}.
    \end{cases}
\end{aligned}
\end{equation*}
Then, the optimal magnet placement is obtained by using the first-order necessary condition as 
\[
L_{k_0}^\ast \approx 
\begin{cases}
1/2.718\ L,\ \text{if } \bar{\mathbf{m}}_{k_0} = \bar{\mathbf{m}}_{k_1}, \\
2/7\ L,\ \text{if } \bar{\mathbf{m}}_{k_0} = -\bar{\mathbf{m}}_{k_1}.
\end{cases}
\]
The result is entirely determined by the structural geometry, independent of material or actuation parameters, and thus exhibits scale invariance. Interestingly, when $\bar{\mb{m}}_{k_0}=\bar{\mb{m}}_{k_1}$, the ratio of optimal placement $L/L_{k_0}^*\approx 2.718$ is numerically close to Euler's number, suggesting a geometric correspondence with the logarithmic spiral structures widely employed in soft continuum design. 

\begin{figure}[t]
    \centering
    \includegraphics[width=\columnwidth]{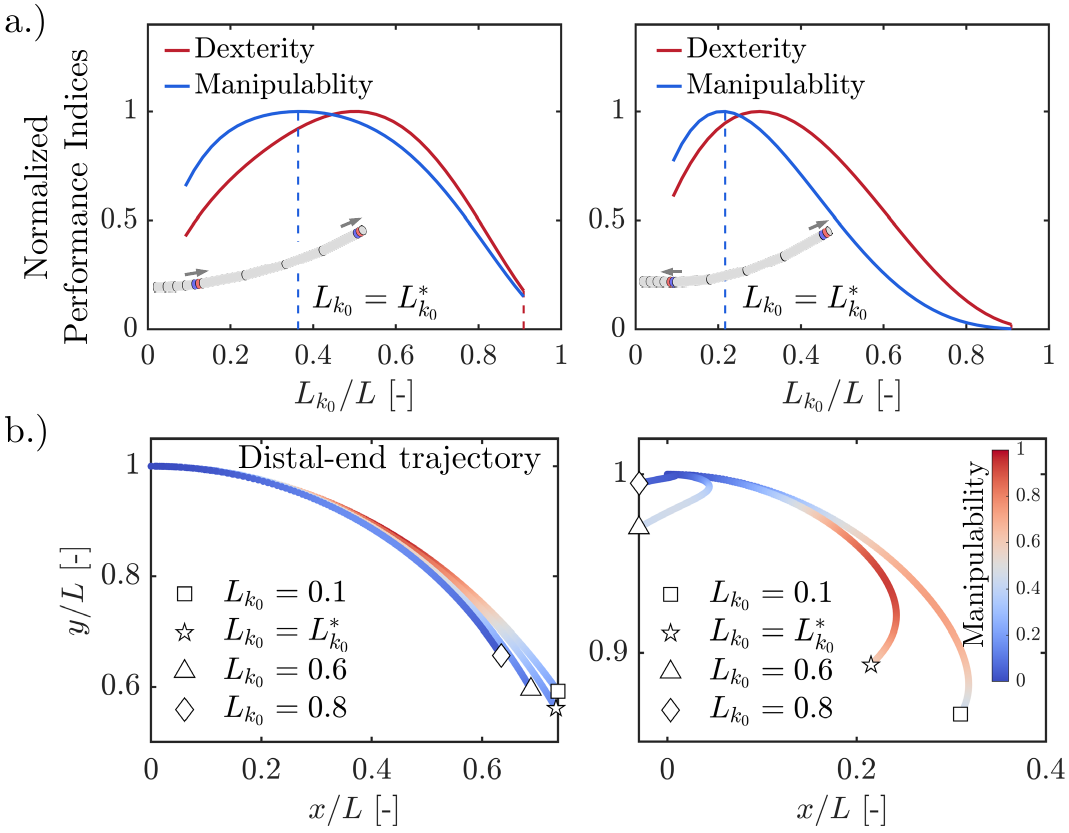}
    \caption{a.) Normalized manipulability and dexterity of MeSCRs with two embedded magnets as functions of normalized magnet position $L_{k_0}/L\in[0, 1]$: (left) aligned magnetic moments, (right) opposing magnetic moments. b.) Distal end trajectories under various magnet placements $L_{k_0}/L$ and moments. Color gradient encodes normalized local manipulability.}
    \label{fig:normalizedindex}
\end{figure}

We also evaluated the kinematic dexterity of the MeSCR by minimizing the distortion index in \cite{parkKinematicDexterityRobotic1994}. 
Since the distortion density satisfies the conditions of Theorem~\ref{thm:optimalresult}, the optimal solutions are expected to concentrate near boundary configurations. 
The numerical results in Fig.~\ref{fig:normalizedindex} confirm this behavior. 
Across both aligned and opposing magnetization modes, the distortion minimizing designs consistently collapse toward the distal end, whereas the manipulability maximizing designs follow the placement trends predicted by the analytical model. 
These outcomes highlight a clear structural organization of the performance landscape. 
Manipulability is governed by the separation and orientation of magnetic sources, while distortion tends to favor extreme placements that suppress curvature variation. 
The reachable trajectories further show that improved local dexterity does not necessarily translate into expanded global coverage. 
This observation reinforces the intrinsic mismatch between local controllability and global reachability and indicates that structural optimization must balance these competing geometric effects rather than optimizing either in isolation.



\subsection{Numerical Results for Multi-magnet Configurations}
For MeSCRs with three embedded magnets, closed-form optimization becomes analytically intractable, thus we employ the gradient-based scheme in Algorithm~\ref{alg:two_level_opt}. 
The numerical experiments (see Fig.~S1) show that the method converges consistently across all orientation patterns and produces solutions that agree with exhaustive or heuristic baselines, which indicates that the objective landscape remains smooth enough for reliable gradient-driven search.

The structural behaviors observed in the two-magnet case persist in the three-magnet configuration. 
More importantly, the aggregated results reveal clear structural tendencies that persist across all configurations. The first two magnets exhibit highly concentrated optimal placements near 
\(L_{k_0}^*/L \approx 0.18\) 
and 
\(L_{k_1}^*/L \approx 0.74\),
reflecting a general preference for distributed magnetic sources that enrich the set of attainable torque directions. Orientation also plays a decisive role. Opposing pairs drive the optimizer to increase their spacing and shift them toward the proximal region, while like-oriented pairs naturally cluster toward one end to exploit larger moment arms. These consistent behaviors show that the interaction between magnet placement and orientation organizes the shape of the actuation space and dominates the resulting optimal designs.

\section{Optimization under General Fields}
\label{sec:V}

In this section, we extend the proposed structural optimization framework to dipole-generated fields and investigate the resulting equilibrium behavior, kinematic performance, and optimal embedded-magnet configurations.

\subsection{Numerical Optimization under Single Dipole Fields}
\begin{table}[t]
\centering
\caption{Geometric and material properties}
\begin{tabular}{lll}
\hline
\rowcolor[HTML]{EFEFEF}
&
\textbf{Parameters} (Description) & \textbf{Value} $[\mbox{Unit}]$ \\ \hline
\multirow{6}{*}{Silicone tube} & $E$ (Young's modulus) & $20$ $[\mathrm{MPa}]$ \\
& $OD$ (Outside diameter) & $1.2$ $[\mathrm{mm}]$ \\
& $ID$ (Inside diameter) & $0.8$ $[\mr{mm}]$ \\
& $L$ (Length) & $40$ $[\mr{mm}]$\\ 
& $\rho$ (Poisson's ratio) & $0.49$ $[\mbox{-}]$ \\
& $N$ (Number of joints) & $8/12$ $[\mbox{-}]$ \\
\hline 
\multirow{5}{*}{Micromagnets} & $B_r$ (Remanence) & $1.2$ $[\mr{T}]$ \\
& $L_m$ (Length) & $5/3/2$ $[\mr{mm}]$ \\ 
& $D_m$ (Diameter) & $1$ $[\mr{mm}]$ \\ 
& $L_g$ (Minimal gap) & $5/3/2$ $[\mr{mm}]$ \\
& $N_m$ (Number of megnets) & $2/3$ $[\mbox{-}]$ \\
\hline
Actuating Magnet & $M_{\mr{b}}$ (Dipole moment) & $342.86$ $[\mr{A\cdot m^2}]$ \\ \hline
\end{tabular}
\label{tab:MSCR1properties}
\end{table}


\begin{figure}[ht]
    \centering
    \includegraphics[width=0.5\textwidth]{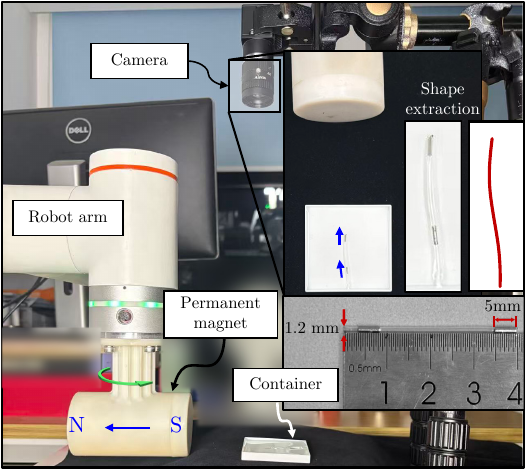}
    \caption{
    Experimental setup: A permanent magnet attached to a 6-DoF robot arm generates a spatially varying dipole field. 
    The MeSCR is placed inside a transparent container, while a downward-facing camera records its deformation for subsequent shape extraction and Jacobian estimation. 
    Insets show details of the MeSCR geometry (outer diameter \SI{1.2}{mm}, embedded magnet length \SI{5}{mm}) and an example of the extracted centerline.
    }
    \label{fig:experimentsetup}
\end{figure}

\begin{figure}[ht]
    \centering
    \includegraphics[width=\linewidth]{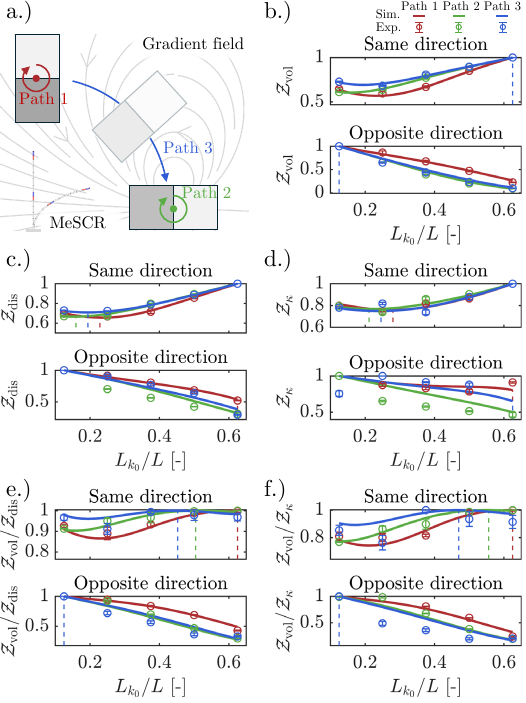}
    \caption{Performance evaluation under three dipole-field task paths.
    a.) Illustration of the three task paths.
    b.)–f.) Normalized kinematic performances with respect to the movable magnet position $L_{k_0}/L$ with same and opposite directions of orientation:
    b.) $\mc{Z}_{\mathrm{vol}}$, 
    c.) $\mc{Z}_{\mathrm{dis}}$, 
    d.) $\mc{Z}_{\kappa}$, 
    e.) $\mc{Z}_{\mathrm{vol}}/\mc{Z}_{\mathrm{dis}}$, 
    and f.) $\mc{Z}_{\mathrm{vol}}/\mc{Z}_{\kappa}$. 
    Solid curves correspond to numerical results for the three paths, and dashed lines indicate the optimal placements. Experimental data are plotted as discrete markers.}
    \label{fig:exptwomag}
\end{figure}

To study the effect of embedded magnet configurations under dipole fields, we evaluate the two principal orientation modes across representative task paths.
The experimental setup is presented in Fig.~\ref{fig:experimentsetup}, and the results comparison to simulations are summarized in Fig.~\ref{fig:exptwomag}, Fig.~S2, and Table~S1, which exhibit consistent trends that align with our analytical insights. 
For the same direction magnetization, all kinematic indices preserve their characteristic ordering from the uniform field analysis. 
The increasing actuation strength toward the distal end leads manipulability to prefer distal placement, while indices sensitive to curvature variation attain their maxima at intermediate spacing. 
These behaviors remain stable across task paths and magnet lengths, indicating that constructive torque accumulation dominates the optimization landscape.

Opposite direction magnetization produces a markedly different pattern. The interaction of opposing torques generates a cancellation zone that governs the deformation geometry, causing all performance measures to collapse toward a narrow proximal region. 
The near invariance across paths and magnet lengths confirms that this behavior is determined primarily by the intrinsic structure of the torque field rather than by task-specific sampling. 
The competing tendencies predicted in the uniform field case also persist: increasing magnet separation strengthens net actuation, whereas proximity favors smoother deformation. 
Under dipole actuation, however, the dominant role of cancellation compresses all optimal solutions toward a common proximal configuration.
These results provide direct experimental validation that magnet orientation is the primary structural factor shaping the Jacobian spectrum and the geometry of the equilibrium manifold. 
Constructive or destructive torque interactions organize the qualitative behavior of all performance measures, while task paths and scaling parameters exert only secondary influence.

\begin{figure}[t]
    \centering
    \includegraphics[width=\columnwidth]{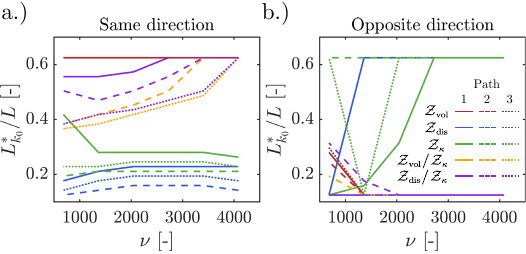}
    \caption{Optimal movable magnet position under varying values of the nondimensional parameter $\nu$. a.) Same direction magnetization. b.) Opposite direction magnetization.}
    \label{fig:varyE}
\end{figure}

We further examine the dependence of optimal placement on the nondimensional parameter
\[
\nu = \frac{4\pi EI}{N_m B_r M_{\mathrm{b}} d^3 A L^2}.
\]
The results in Fig.~\ref{fig:varyE} show that for the same direction magnetization, the optimal positions vary smoothly with \(\nu\) and retain the established ordering. 
For opposite orientation, the optimal placements exhibit larger variation at small \(\nu\) but converge rapidly toward boundary configurations as elasticity increases. 
The presence or absence of internal torque cancellation thus provides the underlying organizing principle for all indices across field strength, task path, and geometry.

\subsection{Numerical Optimization under General Fields}

To evaluate whether the dipole–field behaviors observed in the two magnet case generalize to higher dimensional designs, we extend the analysis to MeSCRs with three embedded magnets, allowing two magnets to vary in both position and orientation. 
The experiments cover multiple dipole configurations, task paths, and magnetic strength parameters \((\nu_1\!-\!\nu_5)\), with results summarized in Fig.~S3. A clear and unified set of structural tendencies emerges across all settings. 
The placement hierarchy mirrors that of the two magnet case: constructive torque accumulation drives distal optimality for \(\mathcal{Z}_{\mathrm{vol}}\), while indices sensitive to curvature modulation, such as \(\mathcal{Z}_{\mathrm{dis}}\) and \(\mathcal{Z}_{\kappa}\), consistently favor proximal configurations. 
Ratio based measures converge to stable mid–distal solutions, indicating path invariant behavior even under complex field variations. 
Orientation exhibits the same role identified previously. 
Aligned moments promote torque amplification and favor more concentrated arrangements, whereas partial spreading becomes advantageous when deformation diversity is required. 
The most consistent layouts arise from the combined indices, which naturally balance these competing trends. 
Together with the two magnet dipole–field results, these findings demonstrate that orientation dependent torque interactions and spacing driven actuation diversity form a unified geometric mechanism that persists across increased design dimensionality and external field complexity.


\section{Conclusion}
\label{sec:VI}

This work introduced a structural optimization framework that unifies equilibrium mechanics with kinematic performance for magnetically actuated soft continuum robots. 
By formulating magnetic loading as an actuation-dependent equilibrium manifold and evaluating Jacobian spectral indices on its pullback geometry, the approach enables physically consistent performance assessment across both analytical and numerically computed configurations. 
The results across uniform and dipole-generated fields reveal that the attainable motion of magnetic continuum structures is fundamentally organized by the interaction pattern of embedded magnets. 
Constructive torque interactions promote distal or mid–distal arrangements that strengthen net actuation, whereas opposing orientations create intrinsic cancellation regions that compress optimal designs toward proximal configurations. 
These tendencies persist across different actuation paths, field strengths, and magnet counts, indicating that orientation-driven torque geometry is a primary determinant of global kinematic behavior.

Certain limitations remain. 
The analytical development relies on small-deflection assumptions and a discrete PRB representation, which may reduce accuracy under strong fields or highly compliant materials.
Magnet–magnet interactions, nonlinear constitutive effects, and dynamic behaviors are not explicitly modeled. 
Extending the approach to richer material descriptions, heterogeneous magnetization, and closed-loop actuation could broaden its practical utility. Further characterization of equilibrium-manifold geometry may deepen understanding of reachable sets, while integrating sensing, real-time control, or data-driven design tools could facilitate systematic structure actuation co-design for future magnetic soft continuum robots.

\bibliographystyle{IEEEtran}  
\bibliography{IEEEabrv, references}

\begin{thebibliography}{10}
\providecommand{\url}[1]{#1}
\csname url@samestyle\endcsname
\providecommand{\newblock}{\relax}
\providecommand{\bibinfo}[2]{#2}
\providecommand{\BIBentrySTDinterwordspacing}{\spaceskip=0pt\relax}
\providecommand{\BIBentryALTinterwordstretchfactor}{4}
\providecommand{\BIBentryALTinterwordspacing}{\spaceskip=\fontdimen2\font plus
\BIBentryALTinterwordstretchfactor\fontdimen3\font minus
  \fontdimen4\font\relax}
\providecommand{\BIBforeignlanguage}[2]{{%
\expandafter\ifx\csname l@#1\endcsname\relax
\typeout{** WARNING: IEEEtran.bst: No hyphenation pattern has been}%
\typeout{** loaded for the language `#1'. Using the pattern for}%
\typeout{** the default language instead.}%
\else
\language=\csname l@#1\endcsname
\fi
#2}}
\providecommand{\BIBdecl}{\relax}
\BIBdecl

\bibitem{Wang2024Sensing}
P.~Wang, Z.~Xie, W.~Xin, Z.~Tang, X.~Yang, M.~Mohanakrishnan, S.~Guo, and
  C.~Laschi, ``Sensing expectation enables simultaneous proprioception and
  contact detection in an intelligent soft continuum robot,'' \emph{Nature
  Communications}, vol.~15, no.~1, nov 2024.

\bibitem{Zhang2025Deformation}
H.~Zhang, J.~Zhang, P.~Xiang, K.~Qiu, Q.~Fang, Y.~Wang, R.~Xiong, and H.~Lu,
  ``Deformation configuration estimation for soft continuum robot utilizing
  seq2seq learning,'' \emph{IEEE Robotics and Automation Letters}, vol.~10,
  no.~12, pp. 13\,280--13\,287, dec 2025.

\bibitem{Zhao2024Controller}
Q.~Zhao, S.~Wang, J.~Hu, H.~Liu, and H.~K. Chu, ``Controller design for a soft
  continuum robot with concurrent continuous rotation,'' \emph{IEEE-ASME
  Transactions on Mechatronics}, vol.~29, no.~6, pp. 4504--4513, dec 2024.

\bibitem{Wang2024Soft}
X.~Wang, Q.~Lu, D.~Lee, Z.~Gan, and N.~Rojas, ``A soft continuum robot with
  self-controllable variable curvature,'' \emph{IEEE Robotics and Automation
  Letters}, vol.~9, no.~3, pp. 2016--2023, mar 2024.

\bibitem{Huang2024Design}
Y.~Huang, Q.~Zhao, J.~Hu, and H.~Liu, ``Design and modeling of a multi-dof
  magnetic continuum robot with diverse deformation modes,'' \emph{IEEE
  Robotics and Automation Letters}, vol.~9, no.~4, pp. 3956--3963, apr 2024.

\bibitem{Park2024Workspace}
J.~Park, H.~Kee, and S.~Park, ``Workspace expansion of magnetic soft continuum
  robot using movable opposite magnet,'' \emph{IEEE Robotics and Automation
  Letters}, vol.~9, no.~7, pp. 6648--6655, jul 2024.

\bibitem{Chathuranga2024Assisted}
D.~Chathuranga, P.~Lloyd, J.~H. Chandler, R.~A. Harris, and P.~Valdastri,
  ``Assisted magnetic soft continuum robot navigation via rotating magnetic
  fields,'' \emph{IEEE Robotics and Automation Letters}, vol.~9, no.~1, pp.
  183--190, jan 2024.

\bibitem{Dreyfus2024Dexterous}
R.~Dreyfus, Q.~Boehler, S.~Lyttle, P.~Gruber, J.~Lussi, C.~Chautems,
  S.~Gervasoni, J.~Berberat, D.~Seibold, N.~Ochsenbein-Kolble, M.~Reinehr,
  M.~Weisskopf, L.~Remonda, and B.~J. Nelson, ``Dexterous helical magnetic
  robot for improved endovascular access,'' \emph{Science Robotics}, vol.~9,
  no.~87, feb 2024.

\bibitem{Zhang2025Kinetostatics}
Y.~Zhang, W.~Song, A.~Li, N.~Xu, X.~Liu, Y.~Zhang, B.~Li, and Y.~Lyu,
  ``Kinetostatics of magnetic captained elastica driven by revolutional and
  rotational magnet,'' \emph{IEEE-ASME Transactions on Mechatronics}, 2025.

\bibitem{Cao2025Magnetic}
Y.~Cao, M.~Cai, B.~Sun, Z.~Qi, J.~Xue, Y.~Jiang, B.~Hao, J.~Zhu, X.~Liu,
  C.~Yang, and L.~Zhang, ``Magnetic continuum robot with modular axial
  magnetization: Design, modeling, optimization, and control,'' \emph{IEEE
  Transactions on Robotics}, vol.~41, pp. 1513--1532, 2025.

\bibitem{Tang2024Learning}
Z.~Tang, W.~Xin, P.~Wang, and C.~Laschi, ``Learning-based control for soft
  robot-environment interaction with force/position tracking capability,''
  \emph{Soft Robotics}, vol.~11, no.~5, pp. 767--778, oct 2024.

\bibitem{Tang2025Learning}
Z.~Tang, P.~Wang, W.~Xin, and C.~Laschi, ``Learning to control a soft robotic
  manipulator under uncertainty and unforeseen changes in robot-environment
  interaction,'' \emph{International Journal of Robotics Research}, 2025.

\bibitem{Kang2025Adaptive}
X.~Kang, F.~Xu, L.~Han, and H.~Wang, ``Adaptive visual servo of soft robot with
  interaction estimation and compensation,'' \emph{IEEE Transactions on
  Automation Science and Engineering}, vol.~22, pp. 18\,395--18\,404, 2025.

\bibitem{Francescon2025Closed}
V.~Francescon, N.~Murasovs, P.~Lloyd, O.~Onaizah, D.~S. Chathuranga, and
  P.~Valdastri, ``Closed-loop shape-forming control of a magnetic soft
  continuum robot,'' \emph{IEEE Robotics and Automation Letters}, vol.~10,
  no.~6, pp. 6071--6078, jun 2025.

\bibitem{VanLewen2025Real}
D.~Van~Lewen, Y.~Lu, F.~Julia-Wise, A.~Vasowalla, C.~Wu, J.~Yeo, E.~Billatos,
  and S.~Russo, ``A real-time, semi-autonomous navigation platform for soft
  robotic bronchoscopy,'' \emph{IEEE Robotics and Automation Letters}, vol.~10,
  no.~5, pp. 4722--4729, may 2025.

\bibitem{Xu2025Automatic}
S.~Xu, B.~Chen, D.~Li, S.~Fu, X.~Wu, S.~Du, and T.~Xu, ``An automatic
  magnetically robotic system using a double-loop stable control method for
  guidewire steering,'' \emph{IEEE-ASME Transactions on Mechatronics}, vol.~30,
  no.~5, pp. 3560--3571, oct 2025.

\bibitem{oreillyModelingNonlinearProblems2017}
O.~M. O'Reilly, \emph{Modeling Nonlinear Problems in The Mechanics of Strings
  and Rods}, ser. Interaction of Mechanics and MaThematics.\hskip 1em plus
  0.5em minus 0.4em\relax Cham: Springer International Publishing, 2017.

\bibitem{roesthuisSteeringMultisegmentContinuum2016}
R.~J. Roesthuis and S.~Misra, ``Steering of multisegment continuum manipulators
  using rigid-link modeling and fbg-based shape sensing,'' \emph{IEEE
  Transactions on Robotics}, vol.~32, no.~2, pp. 372--382, Apr. 2016.

\bibitem{pittiglio2023closed}
G.~Pittiglio, A.~L. Orekhov, T.~da~Veiga, S.~Calo, J.~H. Chandler, N.~Simaan,
  and P.~Valdastri, ``Closed loop static control of multi-magnet soft continuum
  robots,'' \emph{IEEE Robotics and Automation Letters}, vol.~8, no.~7, pp.
  3980--3987, jul 2023.

\bibitem{barfoot2014associating}
T.~D. Barfoot and P.~T. Furgale, ``Associating uncertainty with
  three-dimensional poses for use in estimation problems,'' \emph{IEEE
  Transactions on Robotics}, vol.~30, no.~3, pp. 679--693, 2014.

\bibitem{yoshikawa1985manipulability}
T.~Yoshikawa, ``Manipulability of robotic mechanisms,'' \emph{The International
  Journal of Robotics Research}, vol.~4, no.~2, pp. 3--9, 1985.

\bibitem{parkKinematicDexterityRobotic1994}
F.~C. Park and R.~W. Brockett, ``Kinematic dexterity of robotic mechanisms,''
  \emph{The International Journal of Robotics Research}, vol.~13, no.~1, pp.
  1--15, Feb. 1994.

\bibitem{salisbury1982articulated}
J.~K. Salisbury and J.~J. Craig, ``Articulated hands: Force control and
  kinematic issues,'' \emph{The International Journal of Robotics Research},
  vol.~1, no.~1, pp. 4--17, 1982.

\bibitem{federer2014geometric}
H.~Federer, \emph{Geometric measure Theory}.\hskip 1em plus 0.5em minus
  0.4em\relax Springer, 2014.

\bibitem{linMagneticContinuumRobot2021}
D.~Lin, N.~Jiao, Z.~Wang, and L.~Liu, ``A magnetic continuum robot with
  multi-mode control using opposite-magnetized magnets,'' \emph{IEEE Robotics
  and Automation Letters}, vol.~6, no.~2, pp. 2485--2492, Apr. 2021.

\bibitem{chirikjianStochasticModelsInformation2011}
G.~S. Chirikjian, \emph{Stochastic Models, Information Theory, and Lie Groups,
  Volume 2: Analytic Methods and Modern Applications}.\hskip 1em plus 0.5em
  minus 0.4em\relax Springer Science \& Business Media, Nov. 2011.

\bibitem{lu2023mechanics}
L.~Lu, J.~Sim, and R.~R. Zhao, ``Mechanics of hard-magnetic soft materials: A
  review,'' \emph{MECHANICS OF MATERIALS}, p. 104874, 2023.

\bibitem{Wu2024Closed}
Z.~Wu and J.~Zhang, ``Closed-loop magnetic control of medical soft continuum
  robots for deflection,'' \emph{IEEE-ASME Transactions on Mechatronics}, 2024.

\bibitem{greigarnPseudorigidbodyModelKinematic2015}
T.~Greigarn and M.~C. {\c C}avu{\c s}o{\u g}lu, ``Pseudo-rigid-body model and
  kinematic analysis of mri-actuated catheters,'' in \emph{2015 IEEE
  International Conference on Robotics and Automation (ICRA)}, May 2015, pp.
  2236--2243.

\bibitem{parkWorkspaceExpansionMagnetic2024}
J.~Park, H.~Kee, and S.~Park, ``Workspace expansion of magnetic soft continuum
  robot using movable opposite magnet,'' \emph{IEEE Robotics and Automation
  Letters}, vol.~9, no.~7, pp. 6648--6655, Jul. 2024.

\bibitem{leeSteeringTunnelingStent2021}
W.~Lee, J.~Nam, J.~Kim, E.~Jung, N.~Kim, and G.~Jang, ``Steering, tunneling,
  and stent delivery of a multifunctional magnetic catheter robot to treat
  occlusive vascular disease,'' \emph{IEEE Transactions on Industrial
  Electronics}, vol.~68, no.~1, pp. 391--400, Jan. 2021.

\bibitem{zhangMagneticallyActuatedMicrocatheter2025}
M.~Zhang, L.~Yang, H.~Yang, L.~Su, J.~Xue, Q.~Wang, B.~Hao, Y.~Jiang, K.~F.
  Chan, J.~J.~Y. Sung, H.~Ko, X.~Liu, L.~Wang, B.~Y.~M. Ip, T.~W.~H. Leung, and
  L.~Zhang, ``A magnetically actuated microcatheter with soft rotatable tip for
  enhanced endovascular access and treatment efficiency,'' \emph{Science
  Advances}, Jun. 2025.

\end{thebibliography}

\vspace{-1cm}
\begin{IEEEbiography}[{\includegraphics[width=1in,height=1.25in,clip,keepaspectratio]{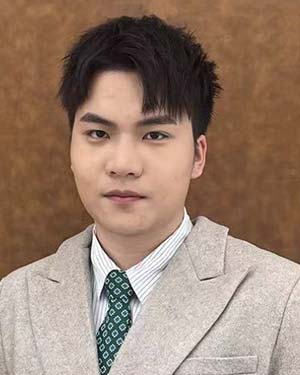}}]{Zhiwei Wu} (Graduate Student Member, IEEE) received the B.E. degree in control science and engineering from the College of Information Science and Technology, Beijing University of Chemical Technology, Beijing, China, in 2022. He is currently working toward the Ph.D. degree in control science and engineering with the Department of Automation, Beijing Institute of Technology, Beijing. His research interests include surgical robotic systems and the magnetic soft continuum robots.
\end{IEEEbiography}

\vspace{-1cm}
\begin{IEEEbiography}[{\includegraphics[width=1in,height=1.25in,clip,keepaspectratio]{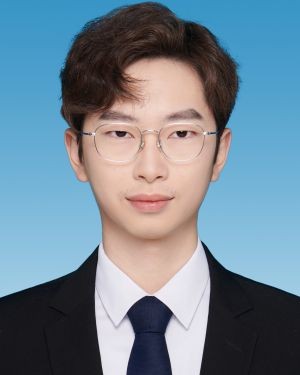}}]{Jiahao Luo} received the M.E. degree in automation from the Guangdong University of Technology, Guangzhou, China, in 2024. He is currently pursuing the Ph.D. degree in control science and engineering at the Beijing Institute of Technology, Beijing, China. His current research interests include surgical robotics and magnetic field localization and control.
\end{IEEEbiography}

\vspace{-1cm}
\begin{IEEEbiography}[{\includegraphics[width=1in,height=1.25in,clip,keepaspectratio]{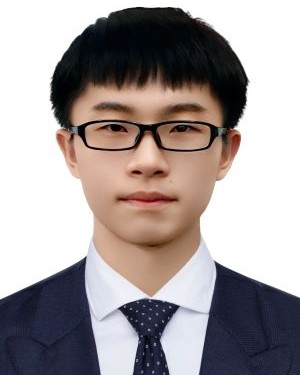}}]{Siyi Wei} received the Ph.D. degree in control science and engineering from the Beijing Institute of Technology, Beijing, China, in 2025.
He joined the Beijing Institute of Technology in 2025, where he is currently a postdoctoral researcher.
His current research interests include surgical robotics and soft robots.
\end{IEEEbiography}

\vspace{-1cm}
\begin{IEEEbiography}
[{\includegraphics[width=1in,height=1.25in,clip,keepaspectratio]{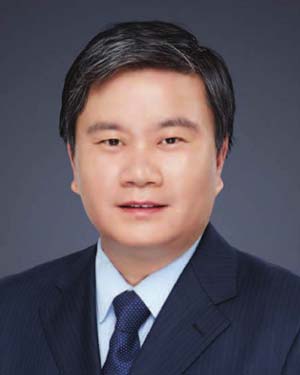}}]{Jinhui Zhang} received the Ph.D. degree in control science and engineering from the Beijing Institute of Technology, Beijing, China, in 2011, where he was an Associate Professor from March 2011 to March 2016. He was a Professor with the School of Electrical and Automation Engineering, Tianjin University, Tianjin, China, from April 2016 to September 2016. He joined the Beijing Institute of Technology in October 2016, where he is currently a Professor. His research interests include control systems and surgical robotics.
\end{IEEEbiography}

\end{document}


\maketitle

\tableofcontents

\newpage
\section{Hessians of the Magnetic Potential Energy}

The Hessian of $\mc{E}_m$ is known as symmetric. It is convenient to present the matrix in parts.
A portion of the Hessian has the following representation:
\begin{equation*}
\begin{aligned}
    \frac{\pt^2\mc{E}_{m}}{\pt\bs\theta_i\pt\bs\theta_i^\top}=\sum_{k\in\mathbb{K},k\geq i}\operatorname{sym}\lt\mb{J}_{\mb{R}_l}^{\bs\theta_i}\mb{R}_{0}^{i^\top}\left[\mb{b}_k\right]_\times^\top\left[\mb{m}_{k}\right]_\times\mb{R}_{0}^{i}\mb{J}_{\mb{R}_r}^{\bs\theta_i}+\mb{Q}_i(\bs\theta_i,\bs\rho_i)\rt\\
\end{aligned}
\end{equation*}
where $\mb{J}_{\mb{R}_l}^{\bs\theta_i}$ is the left Jacobian of the $\mr{SO}(3)$ manifold that \(\mb{J}_{\mb{R}_l}^{\bs\theta_i}=\lt\mb{J}_{\mb{R}_r}^{\bs\theta_i}\rt^\top\), and
\begin{equation*}
    \begin{aligned}
        \mb{Q}_i(\bs\theta_i,\bs\rho_i)&=\lt2\frac{\theta_i-\sin\theta_i}{\theta_i^3}-\frac{1-\cos\theta_i}{\theta_i^2}\rt\lt[\bs\theta_i]_\times[\bs\rho_i]_\times\rt\\
        &-\lt\frac{\theta_i-\sin\theta_i}{\theta_i^3}+2\frac{1-\cos\theta_i-\frac12\theta_i^2}{\theta_i^4}\rt\lt[\bs\theta_i]_\times^2[\bs\rho_i]_\times\rt\\
        &-\left(\frac{1-\cos\theta_i-\frac12\theta_i^2}{\theta_i^4}-3\frac{\theta_i-\sin\theta_i-\frac16\theta_i^3}{\theta_i^5}\right)\\
        &\times\lt[\boldsymbol\theta_i]_\times[\bs\rho_i]_\times[\bs\theta_i]_\times^2\rt\\
    \end{aligned}
\end{equation*}
is the derivative matrix of $\mb{J}_{\mb{R}_l}^{\bs\theta_i}$ with respect to $\bs\theta_i$. Also, $\theta_i=\|\bs\theta_i\|$ denotes the norm of $\bs\theta_i$, $\bs\rho_i=\mb{R}_{0}^{i^\top}\left[\mb{m}_{k}\right]_\times^\top\mb{b}_k$ denotes the residual vector, and $\operatorname{sym}\lt\mb{A}\rt=(\mb{A}+\mb{A}^\top)/2$ finds the symmetric part of any square matrix $\mb{A}$.

The second part of the Hessian is when the index $i<j$, which is expressed as follows:
\begin{equation*}
    \frac{\pt^2\mc{E}_{m}}{\pt\bs\theta_i\pt\bs\theta_j^\top}=\sum_{k\in\mathbb{K},k\geq j}\mb{J}_{\mb{R}_l}^{\bs\theta_i}\mb{R}_{0}^{i^\top}\left[\mb{b}_k\right]_\times^\top\left[\mb{m}_{k}\right]_\times\mb{R}_{0}^{j}\mb{J}_{\mb{R}_r}^{\bs\theta_j}.
\end{equation*}
The third part is when the index $i>j$, expressed by
\begin{equation*}
    \begin{aligned}
        \frac{\pt^2\mc{E}_{m}}{\pt\bs\theta_i\pt\bs\theta_j^\top}&=\sum_{k\in\mathbb{K},k\geq i}\left(\mb{J}_{\mb{R}_l}^{\bs\theta_i}\mb{R}_{0}^{i^\top}\left[\mb{b}_k\right]_\times^\top\left[\mb{m}_{k}\right]_\times\mb{R}_{0}^{j}\mb{J}_{\mb{R}_r}^{\bs\theta_j}-\mb{J}_{\mb{R}_l}^{\bs\theta_i}\mb{R}_{0}^{i^\top}\left[\mb{b}_k\times\mb{m}_{k}\right]_\times^\top\mb{R}_{0}^{j}\mb{J}_{\mb{R}_r}^{\bs\theta_j}\right)\\
        &=\sum_{k\in\mathbb{K},k\geq i}\mb{J}_{\mb{R}_l}^{\bs\theta_i}\mb{R}_{0}^{i^\top}\left[\mb{m}_{k}\right]_\times^\top\left[\mb{b}_k\right]_\times\mb{R}_{0}^{j}\mb{J}_{\mb{R}_r}^{\bs\theta_j}
    \end{aligned}
\end{equation*}
which matches the transpose of the second part, as expected.

\section{Proof of Theorem and Lemma}
\subsection{Preliminary}
The equilibrium equation of the MeSCR can be written as
\begin{equation}
    \label{eqn:stdcontract}
    \bs\theta=\bs\Lambda^{-1}\mb{M}(\bs\theta)\mb{b}.
\end{equation}
\begin{proposition}
    \label{pro:1}
    The matrix-valued function $\mb{M}(\bs\theta)$ and the magnetic torque $\mb{M}(\bs\theta)\mb{b}$ are bounded on $\mathbb{R}^{3N}$ with positive constants $\mc{M}_0$ and $\mc{M}$ such that $\|\mb{M}(\bs\theta)\|\leq\mc{M}_0$ and $\|\mb{M}(\bs\theta)\mb{b}\|\leq\mc{M}$ for all $\bs\theta\in\mathbb{R}^{3N}$, respectively.
\end{proposition}

\begin{proof}
\label{proof:pro1}
Noticing that
\begin{equation*}
    \begin{aligned}
    \|\mb{M}_i(\bs\theta)\mb{b}\|&=\left\|\sum_{k\in\mathbb{K},k\geq i}\lt\left[\mb{m}_{k}\right]_\times^\top\mb{R}_{0}^{i}\mb{J}_{\mb{R}_r}^{\bs\theta_i}\rt^\top\mb{b}(\mb{p}_k)\right\|\\
    &\leq\left\|\mb{R}_{0}^{i}\mb{J}_{\mb{R}_r}^{\bs\theta_i}\right\|\sum_{k\in\mathbb{K},k\geq i}\left\|\left[\mb{m}_{k}\right]_\times\mb{b}(\mb{p}_k)\right\|\\
    &=\sum_{k\in\mathbb{K},k\geq i}M_kB_k\triangleq\mc{M}_i,
    \end{aligned}
\end{equation*}
it is convenient to obtain:
\begin{equation*}
\begin{aligned}
    \|\mb{M}(\bs\theta)\mb{b}\|&=\sqrt{\sum_{i=0}^{N-1}\|\mb{M}_i(\bs\theta)\mb{b}\|^2}\leq\sqrt{\sum_{i=0}^{N-1}\mc{M}_i^2}\triangleq\mc{M},
\end{aligned}
\end{equation*}
and
\begin{equation*}
    \|\mb{M}(\bs\theta)\|\!=\!\sqrt{\sum_{i=0}^{N-1}\|\mb{M}_i(\bs\theta)\|}\!\leq\!\sqrt{N_mk_{N_m-1}}\max_kM_k\triangleq\mc{M}_0.
\end{equation*}
\end{proof}

\begin{proposition}
    \label{pro:2}
    The vector-valued function $\mb{M}(\bs\theta)\mb{b}$ is Lipschitz continuous on $\mathbb{R}^{3N}$ with a Lipschitz constant $\mc{L}$ such that
    $\|\mb{M}(\bs\theta)\mb{b}-\mb{M}(\bs\varphi)\mb{b}\|\leq\mc{L}\|\bs\theta-\bs\varphi\|$ for all $\bs\theta,\bs\varphi\in\mathbb{R}^{3N}$.
\end{proposition}

\begin{proof}
\label{proof:pro2}
Given that the Hessian matrix is symmetric (see the supplementary material for details), it can be expressed equivalently as
\begin{equation*}
    \frac{\pt^2\mc{E}_m}{\pt\bs\theta\pt\bs\theta^\top}=\sum_{k\in\mathbb{K}}\operatorname{sym}\lt\mb{D}_k^\top\mb{C}_k\mb{D}_k+\mb{Q}^k\rt\triangleq\mathbf{S}_m.
\end{equation*}
Whence,
\begin{equation}
    \label{eqn:Dk}
    \mb{D}_k=\operatorname{blkdiag}\{\mb{R}_{0}\mb{J}_{\mb{R}_r}^{\bs\theta_0},\dots,\mb{R}_{0}^{k}\mb{J}_{\mb{R}_r}^{\bs\theta_k},\underbrace{\mb{0}_{3\times3},\dots,\mb{0}_{3\times3}}_{N-k-1}\},
\end{equation}
\begin{equation*}
    \mb{Q}^k=\operatorname{blkdiag}\{\mb{Q}_0,\dots,\mb{Q}_k,\underbrace{\mb{0}_{3\times3},\dots,\mb{0}_{3\times3}}_{N-k-1}\},
\end{equation*}
and
\begin{equation*}
    \mb{C}_k=\operatorname{blkdiag}\{2\mb{U}_{k\times k}-\mb{I}_{k\times k},\underbrace{0\dots0}_{N-k-1}\}\otimes\lt[\mb{b}_{k}]_\times^\top[\mb{m}_k]_\times\rt
\end{equation*}
with $\mb{U}_{k\times k}$ denotes the upper triangle matrix with all elements equal to $1$ and $\otimes$ denotes the Kronecker product. Then, using the triangle inequality and sub-multiplicative property, the following inequality holds for the spectrum norm of the Hessian matrix:
\begin{equation}
    \label{eqn:spectrumOfHessian}
    \begin{aligned}
        \left\|\mb{S}_m\right\|&=\left\|\sum_{k\in\mathbb{K}}\operatorname{sym}\left\{\mb{D}_k^\top\mb{C}_k\mb{D}_k+\mb{Q}^k\right\}\right\|\\
        &\leq\sum_{k\in\mathbb{K}}\left\|\mb{D}_k^\top\right\|\left\|\mb{C}_k\right\|\left\|\mb{D}_k\right\|+\left\|\mb{Q}^k\right\|\\
        &=\sum_{k\in\mathbb{K}}\left\|\mb{C}_k\right\|+\left\|\mb{Q}^k\right\|
    \end{aligned}    
\end{equation}
where we used the fact that $\left\|\mb{D}_k^\top\right\|=\left\|\mb{D}_k\right\|=1$. The spectrum norm of $\mb{C}_k$ is straightforward to be given by
\begin{equation}
\label{eqn:spectrumOfCk}
    \begin{aligned}
        \left\|\mb{C}_k\right\|&=\left\|2\mb{U}_{k\times k}-\mb{I}_{k\times k}\right\|\left\|[\mb{b}_{k}]_\times^\top[\mb{m}_k]_\times\right\|\\
        &\leq\frac{4k}{\pi}\|\mb{b}_k\|\|\bar{\mb{m}}_k\|=\frac{4k}{\pi}M_kB_k
    \end{aligned}
\end{equation}
For small rotation angles, we note that
\begin{equation}
\label{eqn:spectrumOfQk}
    \begin{aligned}
\lim_{\bs\theta\to\mb{0}}\left\|\mb{Q}^k\right\|&=\lim_{\bs\theta\to\mb{0}}\max_{i=0,\dots,k}\left\{\left\|\mb{Q}_i\right\|\right\}=0.
    \end{aligned}
\end{equation}
Combing \eqref{eqn:spectrumOfHessian}, \eqref{eqn:spectrumOfCk}, and \eqref{eqn:spectrumOfQk}, the Jacobian of the vector-valued function $\mb{M}(\bs\theta)\mb{b}$ satisfies:
\begin{equation*}
    \left\|\frac{\pt}{\pt\bs\theta}\mb{M}(\bs\theta)\mb{b}\right\|=\left\|-\mb{S}_m\right\|\leq\sum_{k\in\mathbb{K}}\frac{4k}{\pi}M_kB_k\triangleq\mc{L}.
\end{equation*}
Therefore, the function is Lipschitz continuous on $\mathbb{R}^{3N}$ with the Lipschitz constant $\mc{L}$.

\end{proof}

\begin{theorem}
\label{thm:solunique} 
There exists a unique equilibrium solution to \eqref{eqn:stdcontract} if the following inequality holds:
    \begin{equation}
    \label{eqn:thm1cond}
    \max_{k\in\mathbb{K}} B_k<\frac{\pi}{4}\frac{\lambda_{\min}\lt\bs\Lambda\rt}{\sum_{k\in\mathbb{K}}kM_k}
\end{equation}
where $\lambda_{\min}(\cdot)$ denotes the smallest eigenvalue of a matrix.
\end{theorem}
\begin{lemma}
\label{lmm:eqpointbound}
Let $\bs\theta^*$ be any equilibrium point of \eqref{eqn:stdcontract}. Using Proposition \ref{pro:1}, we have  $\|\bs\theta^*\|=\|\bs\Lambda^{-1}\mb{M}(\bs\theta^*)\mb{b}\|\leq\frac{\mc{M}}{\lambda_{\min}(\bs\Lambda)}$.
\end{lemma}

\subsection{Theorem and Lemma Stated for Proof}

\begin{theorem}
\label{thm:twistfree}
The MeSCR is material-twist-free ($\bs\theta_i^\top\bar{\mb{t}}_i=0$) for any actuating magnetic field, provided that the embedded magnets' magnetic moments are aligned axially, i.e., \(\bar{\mb{m}}_k=\pm M_k\bar{\mb{t}}\), and the magnetic field strength satisfies:
\begin{equation}
\label{eqn:thmcond2}
    \max_{k\in\mathbb{K}} B_k \leq \frac{\sqrt6 \lambda_{\min}(\bs\Lambda)}{\sum_{k\in\mathbb{K}} M_k}.
\end{equation}
\end{theorem}

\begin{theorem}
    \label{thm:DoF}
    The controllable DoF of the MeSCR does not exceed twice the number of embedded magnets.
\end{theorem}

\begin{lemma}
\label{lmm:approximmersionJacobian}
Suppose the inequality condition in Theorem~\ref{thm:solunique} holds. Then the immersion Jacobian admits the expansion
\begin{equation}
\begin{aligned}
\mc{J}_{\bs\theta}^2 &= \det\left(\mb{U}_u^\top\mb{M}^{\top} \bs\Lambda^{-2}\mb{M}\mb{U}_u \right)+\epsilon\\
    &\triangleq\tilde{\mc{J}}_{\bs\theta}^2 + \epsilon
\end{aligned}
\end{equation}
where the remainder term $\epsilon$ is bounded above as
\[
|\epsilon| \leq \frac{3(2\mc{L} - \mc{L}^2)}{(1 - \mc{L})^2} \frac{N_m^3 \mc{M}_0^6}{\Lambda_{\min}^6},
\]
and satisfies $\lim_{B_k \to 0} |\epsilon| \to 0$.
\end{lemma}

\begin{lemma}
\label{lmm:singularity}
Suppose the embedded magnets are placed at the two boundaries (proximal and distal ends), that is $\forall 0\leq N_k<N_m-1$, $L_{k_{N_k}} = 0$ and $L_{k_{N_k+1}} = L$. Then $\tilde{\mc{J}}_{\bs\theta} = 0$.
\end{lemma}

\newpage
\subsection{Proof of Theorem~\ref{thm:twistfree}}
\begin{proof}
The gradient of the potential energy with respect to each joint angle $(\bs\theta_i)$ can be rewritten as:
\begin{equation}
\label{eqn:parEpartheta_i}
    \begin{aligned}
        \nabla_{\bs{\theta}_i}\mc{E}&=\bs\Lambda_i\bs\theta_i-\lt\mb{R}_{0}^{i}\mb{J}_{\mb{R}_r}^{\bs\theta_i}\rt^\top\sum_{k\in\mathbb{K},k\geq i}[\mb{m}_k]_\times\mb{b}_k\\
        &=\bs\Lambda_i\bs\theta_i-\mb{J}_{\mb{R}_l}^{\bs\theta_i}\sum_{k\in\mathbb{K},k\geq i}\mb{R}_{i+1}^{k}\mb{v}_k
    \end{aligned}
\end{equation}
where 
$\mb{v}_k=[\mb{R}_0^{k^\top}\mb{b}_k]_\times^\top\bar{\mb{m}}_k$ is a vector perpendicular to the torsion axis $(\bar{\mb{t}})$ that $\forall k\in\mathbb{K},\mb{v}_k^\top\bar{\mb{t}}=0$. For the $i$th joint, the rotation matrix $\mb{R}_{i+1}^k$ depends on its successor joints. Therefore, we propose a backward iteration method to solve the equilibrium solution.

\textbf{Step 1: Analysis at the boundary joint}

At the boundary joint with index $i=k_{N_m-1}$, where $\mb{R}_{i+1}^k=\mb{I}_3$, we denote $\mb{v}_{k_{N_m-1}}$ by $\mb{v}_k$ for brevity. The left Jacobian associated with $\bs\theta_i$ can be expressed by
\begin{equation}
\label{eqn:leftJacobian}
\mb{J}_{\mb{R}_l}^{\bs\theta_i}=\mb{I}_3+\alpha[\bs\theta_i]_\times+\beta[\bs\theta_i]_\times^2
\end{equation}
where $0<\alpha\leq\frac12$ and $0<\beta\leq\frac16$ are bounded functions of $\bs\theta_i$. Substituting \eqref{eqn:leftJacobian} into \eqref{eqn:parEpartheta_i}, we have:
\begin{equation}
\label{eqn:boundaryparEpartheta}
    \begin{aligned}
       \nabla_{\bs{\theta}_i}\mc{E}&=\bs\Lambda_i\bs\theta_i-\lt\mb{I}_3+\alpha[\bs\theta_i]_\times+\beta[\bs\theta_i]_\times^2\rt\mb{v}_k\\
       &=(\bs\Lambda_i-\alpha[\mb{v}_k]_\times-\beta\bs\theta_i^\top\mb{v}_k\mb{I}_3)\bs\theta_i-\lt1-\beta\|\bs\theta_i\|^2\rt\mb{v}_k.
    \end{aligned}
\end{equation}
The torsional component of \eqref{eqn:boundaryparEpartheta} can be obtained by taking the inner product with $\bar{\mb{t}}$, yielding:
\begin{equation}
\label{eqn:torsionComponents}
\begin{aligned}
    \bar{\mb{t}}^\top\nabla_{\bs{\theta}_i}\mc{E}
       &=\bar{\mb{t}}^\top(\bs\Lambda_i-\alpha[\mb{v}_k]_\times-\beta\bs\theta_i^\top\mb{v}_k\mb{I}_3)\bs\theta_i\\
    &=\lt\lambda_{\min}(\bs\Lambda_i)-\beta\bs\theta_i^\top\mb{v}_k\rt\bar{\mb{t}}^\top\bs\theta_i-\alpha(\bar{\mb{t}}\times\bs\theta_i)^\top\mb{v}_k.
\end{aligned}
\end{equation}

\textbf{Step 2: Solution for $\bs\theta_i^*$ at the boundary joint}

Let the solution of $\bar{\mb{t}}^\top\nabla_{\bs{\theta}_i}\mc{E}=0$ be expressed as $\bs\theta_i^*=s_i\mb{v}_k+\mb{v}_k^\perp$, where $s_i\in\mathbb{R}$ is an undetermined scalar, and $\mb{v}_k^\perp\in\mathbb{R}^3$ is orthogonal to $\mb{v}_k$. Substituting into \eqref{eqn:torsionComponents}, we derive the following system of equations:
\begin{equation*}
    \left\{\begin{aligned}
        &\lt\lambda_{\min}(\bs\Lambda_i)-s_i\beta\|\mb{v}_k\|^2\rt\bar{\mb{t}}^\top\mb{v}_k^\perp=0,\\
        &\alpha\bar{\mb{t}}^\top[\mb{v}_k^\perp]_\times\mb{v}_k=0.
    \end{aligned}\right.
\end{equation*}
From the above, the components of $\bs\theta_i^*$ can be determined as
$s_i = \lambda_{\min}(\bs\Lambda_i)/\lt\beta\|\mb{v}_k\|^2\rt$ and $\mb{v}_k^\perp=\theta_{\mb{t}}^i\bar{\mb{t}}$,
where $\theta_{\mb{t}}^i\in\mathbb{R}$ is arbitrary. We proceed to show that $\theta_{\mb{t}}^i=0$ under certain conditions. From Proposition \ref{pro:1}, one has the inequality that $\|\mb{v}_k\|\leq\mc{M}_k$. Combining with $\beta\leq\frac16$ and by Pythagoras' theorem reads
\begin{equation*}
\|\bs\theta_i^*\|^2=\|s_i\mb{v}_k\|^2+\|\theta_{\mb{t}}^i\bar{\mb{t}}\|^2\geq\lt\frac{6\lambda_{\min}(\bs\Lambda_i)}{\mc{M}_k}\rt^2+\theta_{\mb{t}}^{i^2}.
\end{equation*}
It is known from Lemma \ref{lmm:eqpointbound} that the norm of the equilibrium solution is bounded by $\mc{M}_k/\lambda_{\min}(\bs\Lambda_i)$, we obtain:
\begin{equation}
    \theta_{\mb{t}}^{i^2}\leq\lt\frac{\mc{M}_k}{\lambda_{\min}(\bs\Lambda_i)}\rt^2-\lt\frac{6\lambda_{\min}(\bs\Lambda_i)}{\mc{M}_k}\rt^2.
\end{equation}
Using the condition $\mc{M}_k\leq\sqrt6\lambda_{\min}(\bs\Lambda_i)$ derived from \eqref{eqn:thmcond2}, it follows that $\theta_{\mb t}^{i^2}\leq0$. Therefore, the equilibrium solution simplifies to $\bs\theta_i^*=s_i\mb{v}_k$, whereas $\mb{v}_k^\perp=\mb{0}$. Substituting $\bs\theta_i^*$ into \eqref{eqn:boundaryparEpartheta}, we find $s_i=1/\lambda_{\max}(\bs\Lambda_i)$. 

\textbf{Step 3: Recursive solution for inner joints}

We now consider the next joint from backward, whose joint index $i=k_{N_m-1}-1$. Using an important property of the rotation matrix that $\mb{R}(s\bs\theta)\bs\theta=\bs\theta$ holds for any arbitrary constant $s$, we quickly notice
\begin{equation*}
    \mb{R}_{i+1}(\bs\theta_{i+1}^*)\mb{v}_{k}=\mb{R}_{i+1}(s_{i+1}\mb{v}_k)\mb{v}_k=\mb{v}_k.
\end{equation*}
Thus, the potential energy gradient of the current joint retains the same form as \eqref{eqn:boundaryparEpartheta}, ensuring that $\bs\theta_i^*$ remains parallel to $\mb{v}_k$. Proceeding iteratively, we derive the general equilibrium solution:
\begin{equation}
    \label{eqn:newiteration}
    \bs\theta_i^*=\frac{\sum_{k\geq i}\mb{v}_k}{\lambda_{\max}(\bs\Lambda_i)}=\frac{\sum_{k\geq i}[\mb{R}_{0}^{k^\top}(\bs\theta^*)\mb{b}_k]_\times^\top\bar{\mb{m}}_k}{\lambda_{\max}(\bs\Lambda_i)}.
\end{equation}
Although \eqref{eqn:newiteration} does not yield an explicit closed-form solution for $\bs\theta^*$, it nonetheless ensures that $\bar{\mb{t}}^\top\bs\theta_i^*=0$, since $\mb{v}_k^\top\bar{\mb{t}}=0$. This implies that the MeSCR is material-twist-free.
\end{proof}

\newpage
\subsection{Proof of Theorem~\ref{thm:DoF}}
\begin{proof}
\label{proof:DoF}
    Recall that the actuation Jacobian $\mb{J}_{\mb{b}}=\mb{J}_{\bs\theta}\mb{S}^{-1}\mb{M}(\bs\theta)$, the post-multiplied matrix $(\mb{M}(\bs\theta))$ can be rewritten as
\begin{equation}
\label{eqn:rewriteOfM}
    \begin{aligned}
        \mb{M}(\bs\theta)&=\mb{D}_{k_{N_m-1}}^\top\lt\begin{bmatrix}
            \begin{matrix}
                \mb{1}_{k_0}\\
                \mb{0}_{N-k_0}
            \end{matrix}
            \cdots
            \begin{matrix}
                \mb{1}_{k_{N_m-1}}\\
                \mb{0}_{N-k_{N_m-1}}
            \end{matrix}
        \end{bmatrix}
        \otimes \mb{I}_{3\times 3}\rt\\
        &\times
        \operatorname{blkdiag}\left\{\left[\mb{m}_{k_0}\right]_\times,\dots,\left[\mb{m}_{k_{N_m-1}}\right]_\times\right\}
    \end{aligned}
\end{equation}
in which $\mb{D}_{k_{N_m-1}}$ is provided in \eqref{eqn:Dk} and $\mb{1}_k$ denotes the all-ones vector of $k$ dimensions. To determine the rank of $\mb{M}(\bs\theta)$, one starts by noting that
\begin{equation*}
    \operatorname{rk}\lt\begin{bmatrix}
            \begin{matrix}
                \mb{1}_{k_0}\\
                \mb{0}_{N-k_0}
            \end{matrix}
            \cdots
            \begin{matrix}
                \mb{1}_{k_{N_m-1}}\\
                \mb{0}_{N-k_{N_m-1}}
            \end{matrix}
        \end{bmatrix}\otimes\mb{I}_3\rt=3N_m.
\end{equation*}
Applying Sylvester’s rk inequality, we obtain:
\begin{equation}
\label{eqn:rkupper}
\begin{aligned}
    \operatorname{rk}\lt\mb{M}(\bs\theta)\rt&\leq\min\!\left\{\sum_{i=0}^{N_m-1}\!\operatorname{rk}\!\lt\mb{R}_{0}^{k_i}\mb{J}_{\mb{R}_r}^{\bs\theta_{k_i}}\rt, \sum_{i=0}^{N_m-1}\!\operatorname{rk}\!\lt\left[\mb{m}_{k_i}\right]_\times\rt\right\}\\
    &=2N_m,
\end{aligned}
\end{equation}
and
\begin{equation}
\label{eqn:rkbelow}
\begin{aligned}
    \operatorname{rk}\!\lt\mb{M}(\bs\theta)\rt&\geq\sum_{i=0}^{N_m-1}\!\lt\operatorname{rk}\!\lt\mb{R}_{0}^{k_i}\mb{J}_{\mb{R}_r}^{\bs\theta_{k_i}}\rt+\!\operatorname{rk}\!\lt\left[\mb{m}_{k_i}\right]_\times\rt\rt-3N_m\\
    &=2N_m.
\end{aligned}
\end{equation}
Here we used the facts that $\mb{J}_{\mb{R}_r}^{\bs\theta_{k_i}}$ and $\mb{R}_{0}^{k_i}$ are both invertible and the rank of a skew-symmetric matrix remains $2$. Combined \eqref{eqn:rkupper} with \eqref{eqn:rkbelow} leads to
\(
    \operatorname{rk}\lt\mb{M}(\bs\theta)\rt=2N_m.
\)
Consequently, an important conclusion is derived as
\(\operatorname{rk}\lt\mb{J}_\mb{b}\rt\leq\min\left\{6,\operatorname{rk}\lt\mb{M}(\bs\theta)\rt\right\}=\min\left\{6, 2N_m\right\}\).
\end{proof}

\newpage
\subsection{Proof of Lemma~\ref{lmm:approximmersionJacobian}}
\label{append:lemmaandtheorem}
\begin{proof}
Let $\mb{C} = \mb{M}(\bs\theta)\mb{U}_u$ and $\mb{E} = \mb{C}^\top \bs\Lambda^{-2} \mb{C}$. Then the immersion Jacobian can be expressed as $\mc{J}_{\bs\theta}^2 = \det(\mb{E} + \mb{F})$, where $\mb{F} = \mb{C}^\top (\mb{S}^{-2} - \bs\Lambda^{-2}) \mb{C}$. From Proposition~\ref{pro:1}, the following useful bounds hold:
\begin{equation*}
\label{eqn:preprop}
\|\mb{C}\|\!\leq\!\sqrt{N_m} \mc{M}_0,\ 
\|\mb{E}^{-1}\|\!\leq\!\frac{\Lambda_{\min}^2}{N_m \mc{M}_0^2},\  
\det(\mb{E})\!\leq\!\frac{N_m^3 \mc{M}_0^6}{\Lambda_{\min}^6}.
\end{equation*}
Noticing that the determinant admits the first-order expansion
\[
\det(\mb{E} + \mb{F}) 
= \det(\mb{E}) \left(1 + \tr(\mb{E}^{-1} \mb{F}) + \mathcal{O}(\|\mb{F}\|^2)\right),
\]
which implies the following upper bound:
\begin{equation}
\begin{aligned}
    \label{eqn:epsilon-bound-1}
|\det(\mb{E} + \mb{F}) - \det(\mb{E})|&\lesssim \det(\mb{E})  |\tr(\mb{E}^{-1} \mb{F})| \\
&\leq 3 \det(\mb{E})\|\mb{E}^{-1}\| \|\mb{F}\|.
\end{aligned}
\end{equation}

Define the residual matrix $\mb{G} = \mb{S}^{-1} - \bs\Lambda^{-1}$, it reads $\mb{F} = \mb{C}^\top (\mb{G} \bs\Lambda^{-1} + \bs\Lambda^{-1} \mb{G} + \mb{G}^2) \mb{C}$. From Proposition~\ref{pro:2} and under the condition \(\|\bs\Lambda^{-1} \mb{S}_m\| < 1\), the residual satisfies the following bound derived from the Neumann series:
\[
\|\mb{G}\| \leq \frac{\|\bs\Lambda^{-1} \mb{S}_m\| \|\bs\Lambda^{-1}\|}{1 - \|\bs\Lambda^{-1} \mb{S}_m\|} = \frac{\mc{L}}{\Lambda_{\min}(1 - \mc{L})}.
\]

Hence, we can estimate \(\|\mb{F}\|\) as follows:
\begin{equation}
\label{eqn:normF}
\begin{aligned}
\|\mb{F}\| &\leq \|\mb{C}^\top\| \left( 2 \|\mb{G}\| \|\bs\Lambda^{-1}\| + \|\mb{G}\|^2 \right) \|\mb{C}\| \\
&= N_m \mc{M}_0^2 \frac{2 \mc{L} - \mc{L}^2}{\Lambda_{\min}^2 (1 - \mc{L})^2}.
\end{aligned}
\end{equation}

Substituting \eqref{eqn:normF} into \eqref{eqn:epsilon-bound-1}, we obtain
\[
|\epsilon|=|\det(\mb{E}+\mb{F})-\det(\mb{E})|\leq 3 \frac{N_m^3 \mc{M}_0^6}{\Lambda_{\min}^6} \frac{2 \mc{L} - \mc{L}^2}{(1 - \mc{L})^2},
\]
as claimed. Finally, since $\mc{L} \to 0$ as \(B_k \to 0\), it follows that \(|\epsilon| \to 0\) in this limit.
\end{proof}

\newpage
\subsection{Proof of Lemma~\ref{lmm:singularity}}
\begin{proof}
Under the assumed boundary placement, it is simplified that
\(\bs\Lambda_m \otimes \mb{I}_3 = \Lambda_{k_{N_m - 1}}^m \left( \mb{v} \otimes \mb{I}_3 \right) \left( \mb{v}^\top \otimes \mb{I}_3 \right)
\),
where $\mb{v} = [\mb{0}_{N_k}^\top,\, \mb{1}_{N_m - N_k}^\top]^\top$. Substituting this into the expression for $\tilde{\mc{J}}_{\bs\theta}$, we obtain
\[
\left( \mb{v}^\top \otimes \mb{I}_3 \right) \mb{U}_u \mb{D}_m = \sum_{i = N_k + 1}^{N_m - 1} \left[ \mb{m}_{k_i} \right]_\times = \left[ \sum_{i = N_k + 1}^{N_m - 1} \mb{m}_{k_i} \right]_\times.
\]
Consequently, the rank of the expression for $\tilde{\mc{J}}_{\bs\theta}$ shows that
\begin{equation*}
    \begin{aligned}
&\operatorname{rk}\!\lt\mb{U}_u^\top\mb{D}_m^\top\!\lt\bs\Lambda_m\!\otimes\mb{I}_3\rt\mb{U}_u\mb{D}_m\rt\!\\&=\operatorname{rk}\!\lt\!\lt\mb{v}^\top\!\!\otimes\mb{I}_3\rt\!\mb{U}_u\mb{D}_m\rt=\operatorname{rk}\lt\left[\sum_{i=N_k+1}^{N_m-1}\mb{m}_{k_i}\right]_\times\rt\leq2.
    \end{aligned}
\end{equation*}
That is, the actuation mapping becomes deficient in rank, and we conclude that $\tilde{\mc{J}}_{\bs\theta} = 0$.
\end{proof}

\newpage
\section{Table and Figures}

\subsection{Figure~S1}
\begin{figure*}[ht]
    \centering
    \includegraphics[width=1.0\textwidth]{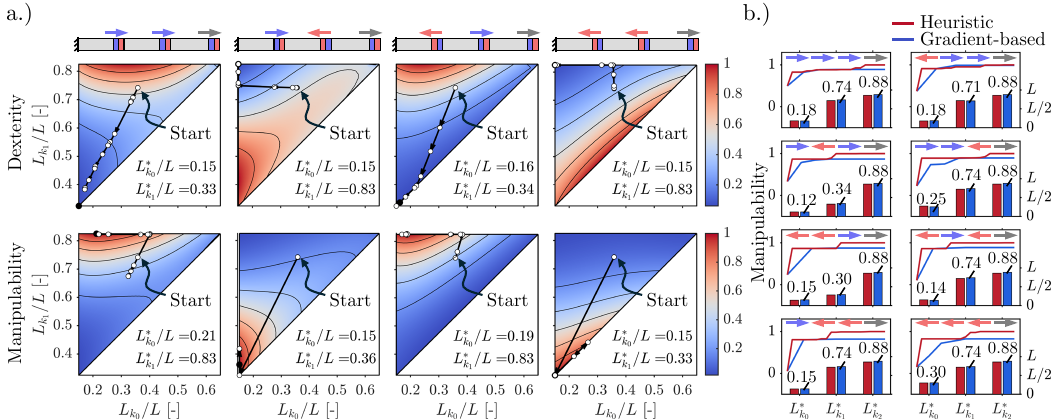}
    \caption{a.) Optimization results under $2^2$ magnetic orientation combinations. The background color maps represent normalized performance values obtained via exhaustive search. Black lines indicate trajectories of the gradient-based optimization. b.) Optimization results across $2^3$ magnetic orientations, comparing the gradient-based method with a heuristic method. Bar plots show the configuration at optimized manipulability with convergence curves overlaid.}
    \label{fig:numericalresult}
\end{figure*}

\newpage
\subsection{Figure~S2}

\begin{figure*}[ht]
    \centering
    \includegraphics[width=\textwidth]{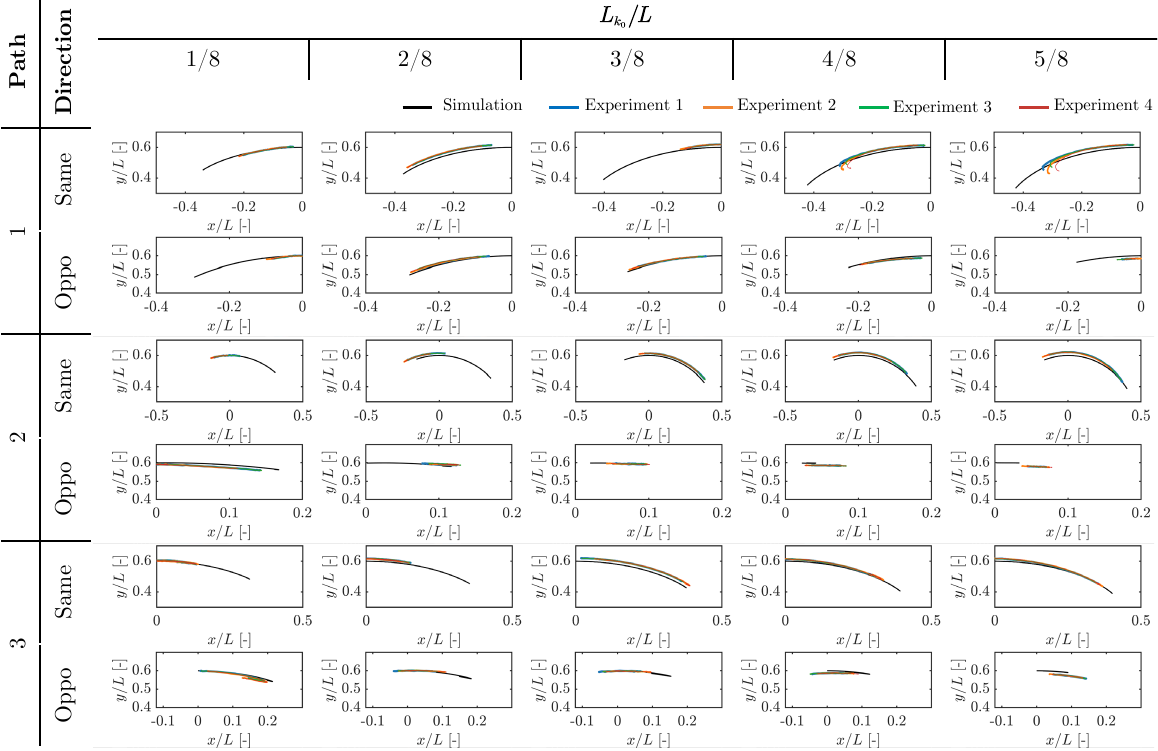}
    \caption{Comparison between simulated and measured end point positions of the MeSCR across different paths, magnetization directions, and movable magnet placements.
    Each panel reports the normalized planar end point position under dipole field actuation for a fixed combination of task path, magnetization direction, and movable magnet position
    $L_{k_0}/L$. The black curve denotes the simulation result, and the colored curves correspond to four independent experimental trials. The data show that the simulated end point positions closely track the measured trajectories over all three paths, with consistent agreement observed across both magnetization directions and across the full range of movable magnet placements. These results confirm that the elastic equilibrium model accurately captures the deformation response of the MeSCR under general dipole field excitation.}
    \label{fig:placeholder}
\end{figure*}

\newpage
\subsection{Table~S1}
\begin{table}[ht]
\centering
\begin{threeparttable}
\caption{Comparison of the optimal placement $L^{\ast}_{k_0}/L$ under three task paths and three magnet lengths $L_m$.}
\label{tab:compLm}
\begin{tabular}{lc|ccccc|ccccc}
\hline
\rowcolor[HTML]{EFEFEF} 
\cellcolor[HTML]{EFEFEF} & \cellcolor[HTML]{EFEFEF} & \multicolumn{5}{c|}{\cellcolor[HTML]{EFEFEF}Same Direction} & \multicolumn{5}{c}{\cellcolor[HTML]{EFEFEF}Opposite Direction} \\ \cline{3-12} 
\rowcolor[HTML]{EFEFEF} 
\cellcolor[HTML]{EFEFEF} & \cellcolor[HTML]{EFEFEF} & \multicolumn{5}{c|}{\cellcolor[HTML]{EFEFEF}$L_{k_0}^*/L$ \textbf{based on kinematic performance:}} & \multicolumn{5}{c}{\cellcolor[HTML]{EFEFEF}$L_{k_0}^*/L$ \textbf{based on kinematic performance:}} \\ \cline{3-12} 
\rowcolor[HTML]{EFEFEF} 
\multirow{-3}{*}{\cellcolor[HTML]{EFEFEF}\begin{tabular}[c]{@{}l@{}}$L_m$\\ $[\mr{mm}]$\end{tabular}} & \multirow{-3}{*}{\cellcolor[HTML]{EFEFEF}\textbf{Path}} & $\mc{Z}_{\mr{vol}}$ & \multicolumn{1}{l}{\cellcolor[HTML]{EFEFEF}$\mc{Z}_{\mr{dis}}$} & $\mc{Z}_{\kappa}$ & $\mc{Z}_{\mr{vol}}/\mc{Z}_{\mr{dis}}$ & $\mc{Z}_{\mr{vol}}/\mc{Z}_{\kappa}$ & \cellcolor[HTML]{EFEFEF}$\mc{Z}_{\mr{vol}}$ & \cellcolor[HTML]{EFEFEF}$\mc{Z}_{\mr{dis}}$ & \cellcolor[HTML]{EFEFEF}$\mc{Z}_{\kappa}$ & \cellcolor[HTML]{EFEFEF}$\mc{Z}_{\mr{vol}}/\mc{Z}_{\mr{dis}}$ & \cellcolor[HTML]{EFEFEF}$\mc{Z}_{\mr{vol}}/\mc{Z}_{\kappa}$ \\ \hline
 & Path 1 & 0.69 & 0.29 & 0.34 & 0.69 & 0.69 & 0.19 & 0.69 & 0.69 & 0.19 & 0.19 \\
 & Path 2 & 0.69 & 0.22 & 0.27 & 0.57 & 0.62 & 0.19 & 0.69 & 0.69 & 0.19 & 0.19 \\
\multirow{-3}{*}{5} & Path 3 & 0.69 & 0.26 & 0.31 & 0.52 & 0.53 & 0.19 & 0.69 & 0.69 & 0.19 & 0.19 \\ \hline
 & Path 1 & 0.76 & 0.31 & 0.35 & 0.76 & 0.72 & 0.16 & 0.76 & 0.76 & 0.16 & 0.16 \\
 & Path 2 & 0.76 & 0.20 & 0.31 & 0.64 & 0.66 & 0.16 & 0.76 & 0.76 & 0.16 & 0.16 \\
\multirow{-3}{*}{3} & Path 3 & 0.76 & 0.25 & 0.31 & 0.56 & 0.58 & 0.16 & 0.76 & 0.76 & 0.16 & 0.16 \\ \hline
 & Path 1 & 0.80 & 0.28 & 0.37 & 0.80 & 0.80 & 0.15 & 0.80 & 0.80 & 0.15 & 0.35 \\
 & Path 2 & 0.80 & 0.19 & 0.28 & 0.80 & 0.78 & 0.15 & 0.80 & 0.80 & 0.15 & 0.15 \\
\multirow{-3}{*}{2} & Path 3 & 0.80 & 0.24 & 0.33 & 0.80 & 0.80 & 0.15 & 0.80 & 0.80 & 0.15 & 0.15 \\ \hline
\textbf{Av.} & -- & 0.75 & 0.25 & 0.32 & 0.68 & 0.69 & 0.17 & 0.75 & 0.75 & 0.17 & 0.19 \\
\textbf{SD} & -- & 0.05 & 0.04 & 0.03 & 0.12 & 0.10 & 0.02 & 0.05 & 0.05 & 0.02 & 0.06 \\
\textbf{CV} & -- & 0.07 & 0.16 & 0.10 & 0.17 & 0.14 & 0.10 & 0.07 & 0.07 & 0.10 & 0.33 \\ \hline
\end{tabular}
\end{threeparttable}
\begin{tablenotes}
    \footnotesize
    \item \textbf{Av.}: Average, \textbf{SD}: Standard Deviation, \textbf{CV}: Coefficient of Variation.
\end{tablenotes}
\end{table}

\newpage
\subsection{Figure~S4}
\begin{figure*}[ht]
    \centering
    \includegraphics[width=\textwidth]{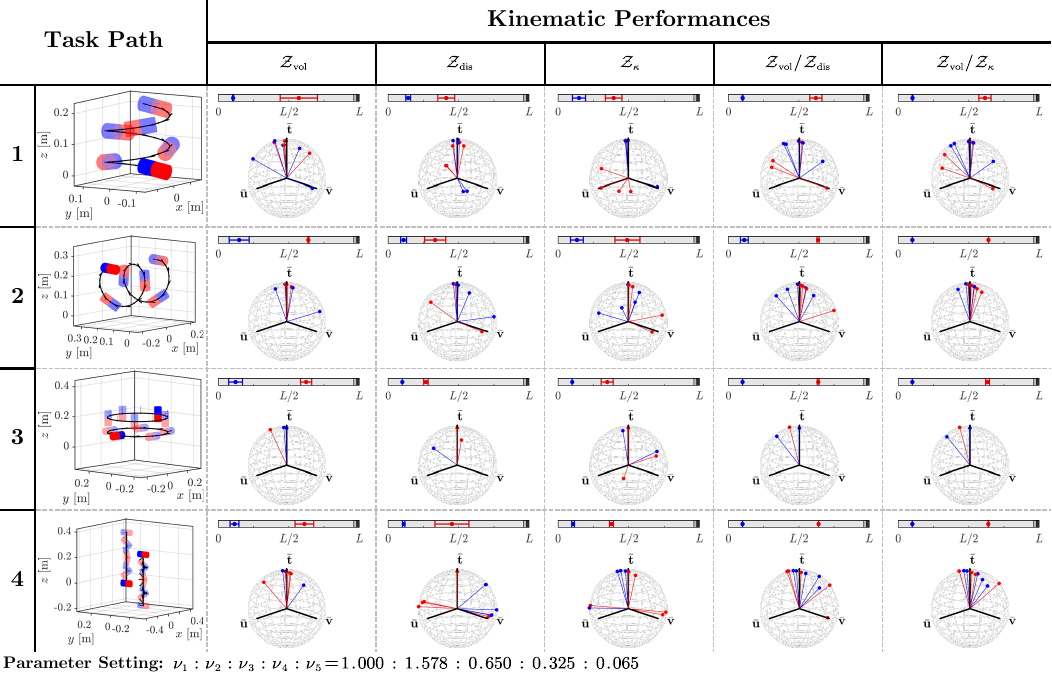}
    \caption{Optimized magnet configurations for MeSCRs with three-embedded magnets under four task paths.
    For each task path and each kinematic performance index, the left subpanel illustrates the dipole-field actuation path, and the right subpanel shows the optimized embedded-magnet layout: horizontal bars indicate the optimal magnet positions along the backbone, and spherical plots depict the corresponding optimal orientations. 
    Blue and red denote the two movable magnets, and black denotes the fixed distal magnet.}
    \label{fig:comprehensivecomp}
\end{figure*}